\documentclass[10pt,fullpage,letterpaper]{article}

\usepackage[utf8]{inputenc} 
\usepackage[T1]{fontenc}    
\usepackage{hyperref}       
\usepackage{url}            
\usepackage{booktabs}       
\usepackage{amsfonts}       
\usepackage{nicefrac}       
\usepackage{microtype}      
\usepackage{xcolor}         
\usepackage{algorithm}
\usepackage{algorithmic}
\usepackage{enumerate}
\usepackage{amsmath, amssymb}
\usepackage{amsthm}
\usepackage{mathrsfs}
\usepackage{bm}
\usepackage{graphicx}
\usepackage{float}
\usepackage{array}
\usepackage{dsfont}
\usepackage{comment}
\usepackage{natbib}
\usepackage{anyfontsize}
\usepackage{todonotes}
\usepackage{wrapfig}
\usepackage{caption}
\usepackage{subcaption}
\usepackage{tikz}
\usepackage{pstricks}
\allowdisplaybreaks

\usepackage{thmtools}
\usepackage{thm-restate}

\hypersetup{
    colorlinks=true,
    linkcolor=blue,
    filecolor=magenta,      
    urlcolor=blue,
}

\DeclareMathOperator*{\argmin}{argmin}
\DeclareMathOperator*{\argmax}{argmax}
\DeclareMathOperator*{\arginf}{arginf}

\def\1{\mathds{1}}
\def\E{\mathbb{E}}
\def\P{\mathbb{P}}
\def\R{\mathbb{R}}
\def\X{\mathcal{X}}
\newcommand{\htheta}{\widehat{\theta}}
\newcommand{\otheta}{\overline{\theta}}
\newcommand{\hatone}{\widehat{(1)}}

\newcommand{\hdeltat}{\widehat{\Delta}^{(t)}}

\newcommand{\mc}[1]{\mathcal{#1}}
\newcommand{\Sp}[1]{\left(#1\right)}
\newcommand{\Mp}[1]{\left[#1\right]}
\newcommand{\Bp}[1]{\left\{#1\right\}}
\newcommand{\abs}[1]{\left|#1\right|}
\newcommand{\Norm}[1]{\left\|#1\right\|}
\newcommand{\ve}[1]{\mathbf{#1}}
\newcommand{\floor}[1]{\left\lfloor#1\right\rfloor}
\newcommand{\ceil}[1]{\left\lceil#1\right\rceil}
\newcommand{\inner}[1]{\left\langle#1\right\rangle}
\newcommand{\matenv}[1]{\left[\begin{matrix}#1\end{matrix}\right]}

\newtheorem{lemma}{Lemma}
\newtheorem{corollary}{Corollary}
\newtheorem{theorem}{Theorem}

\theoremstyle{definition}

\theoremstyle{remark}
\newtheorem{remark}{Remark}

\makeatletter
\newcommand{\IfTwoColumnElse}[2]{%
    \if@twocolumn
        #1 
    \else
        #2 
    \fi
}
\makeatother

\title{A/B Testing and Best-arm Identification for Linear Bandits with Robustness to Non-stationarity}
\author{Zhihan Xiong\footnote{Equal contribution.}\\ \url{zhihanx@cs.washington.edu} \and Romain Camilleri\footnotemark[\value{footnote}] \\ \url{camilr@cs.washington.edu} \and Maryam Fazel \\ \url{mfazel@uw.edu} \and Lalit Jain \\\url{lalitj@uw.edu} \and Kevin Jamieson \\ \url{jamieson@cs.washington.edu}}
\date{}

\begin{document}
\maketitle

\renewcommand*{\thefootnote}{\arabic{footnote}}

\begin{abstract}

We investigate the fixed-budget best-arm identification (BAI) problem for linear bandits in a potentially non-stationary environment. Given a finite arm set $\mathcal{X}\subset\mathbb{R}^d$, a fixed budget $T$, and an unpredictable sequence of parameters $\left\lbrace\theta_t\right\rbrace_{t=1}^{T}$, an algorithm will aim to correctly identify the best arm $x^* := \arg\max_{x\in\mathcal{X}}x^\top\sum_{t=1}^{T}\theta_t$ with probability as high as possible. Prior work has addressed the stationary setting where $\theta_t = \theta_1$ for all $t$ and demonstrated that the error probability decreases as $\exp(-T /\rho^*)$ for a problem-dependent constant $\rho^*$. But in many real-world $A/B/n$ multivariate testing scenarios that motivate our work, the environment is non-stationary and an algorithm expecting a stationary setting can easily fail. For robust identification, it is well-known that if arms are chosen randomly and non-adaptively from a G-optimal design over $\mathcal{X}$ at each time then the error probability decreases as $\exp(-T\Delta^2_{(1)}/d)$, where $\Delta_{(1)} = \min_{x \neq x^*} (x^* - x)^\top \frac{1}{T}\sum_{t=1}^T \theta_t$. As there exist environments where $\Delta_{(1)}^2/ d \ll 1/ \rho^*$, we are motivated to propose a novel algorithm \textsf{P1-RAGE} that aims to obtain the best of both worlds: robustness to non-stationarity and fast rates of identification in benign settings. We characterize the error probability of \textsf{P1-RAGE} and demonstrate empirically that the algorithm indeed never performs worse than G-optimal design but compares favorably to the best algorithms in the stationary setting.

\end{abstract}

\setcounter{footnote}{0} 
\section{INTRODUCTION}
\label{sec:intro}


Data-driven decision-making and A/B testing enable businesses to evaluate strategies using real-time customer data, offering insights into customer tendencies. 
As the use of these methods has increased, these technologies are being utilized to determine problems with smaller effect sizes, while also targeting smaller audiences. These two competing trends of smaller effect sizes and smaller sample sizes make it increasingly challenging to obtain  statistical significance and correct inference since the absolute number of observations is limited. 
Consequently, there is a rising trend in using \emph{adaptive} sampling like multi-armed bandits to obtain the same statistical insights using fewer total observations.

However, using adaptive experimentation schemes can come with many pitfalls. Most algorithms that are effective in practice  (e.g., Thompson Sampling) are developed with the assumption that the \emph{environment is stationary} and that rewards from treatments are stochastic. However in practice this is far from the case. Non-stationarity can be introduced from a variety of sources such as user populations that change from hour to hour, customer preferences which vary over the course of a year, changes in one part of a platform that lead to latency and higher bounceback, site-wide promotions and sales, interference from competitors, macroeconomic shifts, and many other disruptions. 
Many of these issues are often totally unobservable, 
and therefore cannot be controlled, modeled, or accounted for by an experimenter. 
Under such an environment, it is also possible for the underlying performance of treatments to wildly change, and as a result, the treatment that is best performing on any given day may change. 
This makes the concept of ``the best-performing arm'' poorly defined.

Instead, in time-varying settings, the goal of an experimenter is to identify the ``counterfactual best treatment'' at the end of the experimentation period.
That is, the treatment that would have received the \emph{highest total reward had received all the samples}. 
However, in the absence of being able to predict or model time-variation, predicting precisely how a treatment would behave at every time point, at which time at most one treatment can be evaluated, is impossible.
Fortunately, randomization is a powerful tool to provide the next best thing: unbiased \emph{estimates} of how a treatment would behave as if it had been used at every time in the past. 
These methods are well-understood in the causal-inference and online learning literature and are commonly known as inverse-propensity score (IPS) estimators.
The idea is simple: consider a sequence of evaluations from $n$ treatments at each time $\{ x_t \}_{t=1}^T \subset \R^n$. Note that a procedure can only observe at most one treatment per time denoted as $I_t \in [n]$, which is drawn from a distribution $p_t$ over the $n$ treatments. Then $\widehat{X}_i = \frac{1}{T} \sum_{t=1}^T \frac{ \1\{ I_t = i\} }{p_{t,i}} x_{t,i}$ is an unbiased estimator of the cumulative gain $\frac{1}{T}\sum_{t=1}^T x_{t,i}$ by

\IfTwoColumnElse{
    \begin{equation}
        \label{equ:usual_ips}
        \begin{split}
            \E\left[ \frac{ \1\{ I_t = i\} }{p_{t,i}} x_{t,i} \right] =& \sum_{j=1}^n \P(I_t = j) \frac{ \1\{ j = i\} }{p_{t,i}} x_{t,i} \\
            =& \sum_{j=1}^n p_{t,j} \frac{  \1\{ j = i\}  }{p_{t,i}} x_{t,i} = x_{t,i},
        \end{split}
    \end{equation}
}{
    \begin{equation}
        \label{equ:usual_ips}
        \E\left[ \frac{ \1\{ I_t = i\} }{p_{t,i}} x_{t,i} \right] = \sum_{j=1}^n \P(I_t = j) \frac{ \1\{ j = i\} }{p_{t,i}} x_{t,i}= \sum_{j=1}^n p_{t,j} \frac{  \1\{ j = i\}  }{p_{t,i}} x_{t,i} = x_{t,i},
    \end{equation}
}

as long as $\min_{t,i} p_{t,i} > 0$.
Of course, there is no free lunch, and the variance of $\widehat{X}_i$ behaves like $\frac{1}{T^2} \sum_{t=1}^T 1/p_{t,i}$. 
Intuitively, to maximize efficiency of the samples we do take for inference, we should try to minimize the probabilities on poor performing treatments and prioritize mass for the high performing treatments. 
However, if the treatment performances vary over time, it can be challenging to determine how one might do this optimally.
Fortunately, \citet{abbasi2018best} proposes a novel solution to defining these probabilities in a dynamic way that achieves a ``Best of Both Worlds'' (BOBW) guarantee, which is an algorithm called \textsf{P1} that manages to achieve near-optimal rates regardless of whether the environment is stochastic or arbitrarily non-stationary (adversarial).
This seminal work is the gold standard for A/B testing in unpredictable non-stationary settings.


If the number of treatments is small (<10 in practice), BOBW provides a robust solution for practitioners. 
However, there are many situations that practitioners are interested in for which the number of treatments is very large and intractable for traditional A/B testing. For example, multivariate testing \citet{hill2017efficient} aims to identify not just a single best item, but a set or bundle of items, such as the best 6 pieces of content to highlight on a home screen. 
Given $n$ possibilities, this results in $\binom{n}{6}$ total distinct treatments for the A / B test!
Given this combinatorial explosion, practitioners have made structural parametric assumptions, such as the expected value of a set of items behaves like
\begin{align*}
    \theta^{(0)} + \sum_{i=1}^n \theta_i^{(1)} \alpha_i  + \sum_{i=2}^n \sum_{j < i} \theta_{i,j}^{(2)} \alpha_i \alpha_j,
\end{align*}
where $\alpha \in \{ 0, 1\}^n$ with $\sum_i \alpha_i =6$ indicates whether an item was included in the set or not. 
Note that these sums can be succinctly written as $\langle x, \theta \rangle$ for $\theta = (\theta^{(0)},\theta^{(1)},\theta^{(2)})^\top \in \R^{1+n+\binom{n}{2}}$ and an appropriate $x \in \{0,1\}^{1+n+\binom{n}{2}}$. 
This can reduce the overall number of unknowns, and dimension, to just $O( n^2)$ versus $O(n^6)$.
But now the vectors $x \in \mc{X}$, each associated with a particular bundle, are overlapping and can share information.
A similar situation arises if we have features or covariates that describe each possible treatment. 
For example, a particular song comes with lots of metadata including artist, genre, beats per minute, etc. which can encode the useful properties about the song. 

In these kinds of scenario---whether it be multivariate testing or items with feature descriptions---we would like to perform adaptive experimentation in the presence of time-variation. 
Recall that without covariates, we have solutions like \textsf{P1} that are near-optimal for time-variation.
And without time-variation, there are many methods that take covariates into account and are known to be near-optimal.
This work aims to develop an algorithmic framework for handling covariates with time variation.

The remainder of the paper is organized as follows. We discuss the related work in Section \ref{sec:related_work} and presents detailed problem formulations in Section \ref{sec:preliminary}. In Section \ref{sec:nonstationary}, we propose a simple algorithm for general non-stationary environments and then in Section \ref{sec:bobw}, we propose a robust algorithm that can simultaneously tackle stationary and non-stationary environments. Experiment results are presented in Section \ref{sec:experiments} and our conclusions in Section \ref{sec:conclusion}.
  




\section{RELATED WORK}
\label{sec:related_work}

The problem of identifying the best arm in linear bandits is a well-established and extensively researched problem. 
\citep{soare2014best, karnin2016verification, xu2018fully, fiez2019sequential, katz2020empirical, degenne2020gamification, jedra2020optimal, wagenmaker2023instance}. Notably, \citet{katz2020empirical, azizi2021fixed, yang2021towards} focus on the fixed-budget setting and are closely related to our paper. One notable limitation of these algorithms is their reliance on (unrealistic) stationary settings, which leads to their critical failure when applied in non-stationary scenarios. This motivated increasing interest in studying models for non-stationarity in bandits problems and algorithms agnostic to non-stationary settings, which we review next.

\textbf{Models for non-stationarity in bandits. } 
A reasonable approach in bandit problems with distribution shifts is to provide tight models for unknown variations in the reward distribution. Most literature in this setting focuses on minimizing the dynamic regret, which compares the reward obtained against the reward of the best arm in each round $t$. \cite{garivier2011upper} demonstrates that existing methods such as \citet{auer2002nonstochastic} could achieve a dynamic regret of $\widetilde{O}(\sqrt{LT})$ when $L$, the number of distribution shifts, is known. Then, \citet{auer2019adaptively} makes a significant advancement by introducing an adaptive approach with the same dynamic regret but without the knowledge of $L$. More recently, \cite{chen2019new, wei2021non} establish analogous results in the contextual bandits settings. 
Measures of non-stationarity other than $L$ are also considered. In particular, \citet{chen2019new} measures the non-stationarity by total variation and \citet{suk2022tracking} proposes the novel notion of 
severe shifts. 
Note importantly that while this extensive body of work focuses on building tight models of non-stationarity and developing regret minimization algorithms tuned to them, our work is agnostic to such models. 
\textbf{Agnostic non-stationary bandits (Best of both worlds).}
\citet{bubeck2012best,seldin2014one,seldin2017improved,auer2016algorithm,abbasi2018best, lee2021achieving} focus on the ``best of both worlds'' (BOBW) problem: design a bandit algorithm that agnostically achieves optimal performance in both stationary and non-stationary scenarios, even without prior knowledge of the environment.
While most BOBW work focus on regret minimization goals, \citet{abbasi2018best} focuses on BOBW for best-arm identification. In this work, as in \citet{abbasi2018best}, we focus on the agnostic setting.

\textbf{A/B testing.} As discussed in the introduction, our work is closely related to non-stationary A/B testing. In settings with non-stationarity and adaptive sample allocations, non-stationarity can lead to Simpson's paradox if the sample means are used to estimate arm means \citet{kohavi2011unexpected}. It is common in large-scale industrial platforms to assume that means vary smoothly \cite{wu2022non}, or that the differences between them are constant; i.e., all arms are subject to the same random exogeneous shock \cite{stats-accelerator}. The recent work \citet{qin2022adaptivity} models time-variation as arising from confounding due to a context distribution and aims to find the arm with the best reward on average under this context distribution. Their goal is similar to ours, but, unlike them, we do not assume a context distribution.




\section{PRELIMINARIES}
\label{sec:preliminary}
\textbf{Notation.} Let $[a:b]=\Bp{a, a+1, \dots, b}$ for $a, b\in\mathbb{N}$ with $b>a$ and $[a]=\Bp{1, \dots, a}$. For  a vector $x\in\R^d$ and symmetric positive semi-definite (PSD) matrix $A\in\mathbb{S}_+^d$, we use $\Norm{x}_A=\sqrt{x^\top Ax}$ to denote the Mahalanobis norm. For a finite set $\X\subset\R^d$ and distribution $\lambda\in\triangle_{\X}$ over $\X$, we use $A(\lambda)=\E_{x\sim\lambda}\Mp{xx^\top}$ to denote the covariance matrix under $\lambda$.

\subsection{Linear Bandits Problem Formulation}
\label{sec:linear_formulation}

\textbf{General stationary/non-stationary environments.} In this paper, we assume a standard stationary/non-stationary linear bandits model with fixed horizon $T$. 
In particular, let $\X\subset\R^d$ be a finite arm set with $\abs{\X}=K$ such that $\mathrm{span}(\X)=\R^d$. At each time $t=1, \dots, T$, the learner will pick some arm $x_t\in\X$ and receive some noisy reward $r_t=x_t^\top\theta_t+\epsilon_t$, where $\epsilon_t\in[-1, 1]$ is some independent zero-mean noise. All parameters $\Bp{\theta_t}_{t=1}^{T}$ are chosen and fixed by the environment before the game starts.\footnote{Theoretically, this non-stationary setting has no essential difference with the adversarial setting. We choose this non-stationary setting mainly to keep our presentation concise.} The ultimate goal of the learner is to find the optimal arm $\argmax_{x\in\X}x^\top\otheta_T$, where $\otheta_T=\frac{1}{T}\sum_{t=1}^{T}\theta_t$ is the average parameter. This protocol is summarized in Figure \ref{fig:bandits_protocal}.
\begin{figure}[ht]
    \centering
    \fbox{
        \centering
        \begin{minipage}{0.95\linewidth}
            \vspace*{.025in}
            \textbf{Input:} time horizon, $T$; arm set, $\X\subset\R^d$\\
            \textbf{For} $t=1, \dots, T$\\
            \hspace*{.25in} The learner plays arm $x_t\in\X$\\
            \IfTwoColumnElse{
                \hspace*{.25in} The learner receives reward $r_t=x_t^\top\theta_t+\epsilon_t$,\\
                \hspace*{.3in} where $\epsilon_t$ is independent zero-mean noise\\
            }{
                \hspace*{.25in} The learner receives reward $r_t=x_t^\top\theta_t+\epsilon_t$, where $\epsilon_t$ is independent zero-mean noise\\
            }
            The learner recommends arm $x_{J_T}$ 
        \end{minipage}
        }
    \caption{General protocol of fixed-budget best-arm identification (BAI) for linear bandits.}
    \label{fig:bandits_protocal}
\end{figure}

For simplicity, we further assume that $\forall t\in[T]$, $\forall x\in\X$, $x^\top\theta_t\in[-1, 1]$ and the optimal arm $\argmax_{x\in\X}x^\top\otheta_T$ is unique. Meanwhile, similar to \citet{abbasi2018best}, we use the subscript $(k)$ to denote the index of $k$-th best arm in $\X$, which means to have $x_{(1)}^\top\otheta_T> x_{(2)}^\top\otheta_T\geq\dots\geq x_{(K)}^\top\otheta_T$. For each arm $k\in[K]$, we define its gap $\Delta_k$ as
$$\Delta_k=\begin{cases}
    (x_{(1)}-x_k)^\top\otheta_T &\text{if }k\neq (1),\\
    (x_{(1)}-x_{(2)})^\top\otheta_T & \text{if }k=(1).
\end{cases}$$
That is,we have $\Delta_{(1)}=\Delta_{(2)}\leq\Delta_{(3)}\leq\dots\leq\Delta_{(K)}$. As a slight abuse of notation, for unindexed arm $x\in\X$, we will use $\Delta_x$ to denote the gap of $x$. The performance of the learner is measured by its error probability $\P_{\otheta_T}\Sp{J_T\neq (1)}$, where $J_T$ is the index of the learner's recommendation and the probability measure is taken over the randomness inside the learner and the reward noise. Finally, we note that when the setting is stationary, we simply have $\theta_1=\dots=\theta_T=\theta^*$ and everything else is then defined accordingly.

\begin{remark}[Comparison to the adversarial setting]
    The traditional oblivious adversarial setting can be viewed as a special case of our non-stationary setting, in which we simply pick $\epsilon_t=0$ for all $t$ \citep{abbasi2018best}.
\end{remark}




\subsection{BAI for Linear Bandits in Stationary Environments}
\label{sec:stationary}

In this section, we briefly review the well-studied best-arm identification problem for linear bandits in stationary settings. This problem's complexity, first proposed in \citet{soare2014best}, is defined as
\begin{equation}
    \label{equ:rho_star}
    \rho^*(\theta)=H_{\mathsf{LB}}(\theta)=\inf_{\lambda\in\triangle_{\X}}\max_{x\neq x_{(1)}}\frac{\Norm{x-x_{(1)}}^2_{A(\lambda)^{-1}}}{\Delta_x^2},
\end{equation}
where the optimal arm index $(1)$ and gaps $\Delta_k$ are defined based on the input parameter $\theta$. As discussed in \citet{soare2014best}, this complexity is approximately equal to the number of samples required (up to logarithmic terms) to find the best arm by running an oracle algorithm. Later in \citet{fiez2019sequential}, this complexity is proved to be the optimal sample complexity that a BAI algorithm can possibly achieve in a fixed-confidence setting. Recently, \citet{katz2020empirical} proposes algorithm \textsf{Peace} in fixed-budget setting that achieves error probability $\P_{\theta}\Sp{J_T\neq (1)}\leq\widetilde{O}\Sp{\exp\Sp{-\frac{T}{\rho^*(\theta)\log(d)}}}$.\footnote{Rigorously speaking, the error probability of \textsf{Peace} contains another complexity term called $\gamma^*(\theta)$, which is defined as the minimum of a Gaussian width term. However, as argued in \citet{katz2020empirical}, $\gamma^*(\theta)$ is roughly in a same order of $\rho^*(\theta)$. } 


\section{BAI FOR LINEAR BANDITS IN GENERAL NON-STATIONARY ENVIRONMENTS}
\label{sec:nonstationary}

In this section, we present a simple algorithm \textsf{G-BAI} for the general non-stationary environment and analyze its theoretical guarantee. The algorithm is based on the G-optimal design, which refers to the distribution $\lambda^*\in\triangle_{\X}$ such that
\begin{equation}
    \label{equ:g_design}
    \lambda^*=\arginf_{\lambda\in\triangle_{\X}}\max_{x\in\X}\Norm{x}^2_{A(\lambda)^{-1}}.
\end{equation}
Intuitively, G-optimal design allows us to estimate unknown parameter $\theta_t$ uniformly well over all directions of the arms in $\X$ \citep{soare2014best}. which is suitable for addressing non-stationarity since $\theta_t$ may change arbitrarily and each $x\in\X$ may become the optimal at time $t$. Meanwhile, to make sure the estimation of $\theta_t$ is unbiased in a non-stationary environment, we use an IPS estimator. 

Therefore, briefly speaking, at each time $t$, \textsf{G-BAI} simply samples $x_t$ based on G-optimal design and estimate $\theta_t$ through an IPS estimator, whose details are summarized in Algorithm \ref{algo:gbai}.\footnote{We can see $\widehat{\theta}_T$ exactly becomes the more commonly-seen IPS estimator examined in Eq. \eqref{equ:usual_ips} if we apply it to the multi-armed bandits setting, in which we have $K=d$ arms and $\mc{X}=\Bp{\ve{1}_1, \dots, \ve{e}_d}$.}

\begin{algorithm}[ht]
    \caption{G-optimal Best-arm Identification (G-BAI)}
    \label{algo:gbai}
    \begin{algorithmic}[1]
        \REQUIRE budget, $T\in\mathbb{N}$; arm set $\mc{X}\subset\R^d$
        \STATE Compute G-optimal design $\lambda^*$ based on Eq. \eqref{equ:g_design}
        \FOR{$t=1, 2, \dots, T$}
            \STATE Sample $x_t\sim\lambda^*$ and receive reward $r_t$
        \ENDFOR
        \STATE Estimate $\widehat{\theta}_T\leftarrow \frac{1}{T}\sum_{t=1}^T\E_{x\sim\lambda^*}\Mp{xx^\top}^{-1}x_t r_t$
        \RETURN $\argmax_{x\in\X}x^\top\widehat{\theta}_T$
    \end{algorithmic}
\end{algorithm}

By the famous Kiefer-Wolfowitz theorem, an important property of the G-optimal design is that $\max_{x\in\X}\Norm{x}^2_{A(\lambda^*)^{-1}}=d$ \citep{lattimore2020bandit}. With this property, the variance of estimator $\htheta_t$ can be easily controlled. We can then bound the error probability of \textsf{G-BAI} by this fact and the result is summarized in the following theorem.

\begin{restatable}[Error probability of \textsf{G-BAI}]{theorem}{advupperbound}
    \label{theo:adv_upper_bound}
    Fix time horizon $T$, arm set $\X\subset\R^d$ with $\abs{\X}=K$ and arbitrary unknown parameters $\Bp{\theta_t}_{t=1}^T$. If we run Algorithm \ref{algo:gbai} in this non-stationary environment and obtain $x_{J_T}$, then it holds that
    \IfTwoColumnElse{
        \begin{align*}
            &\P_{\otheta_T}\Sp{J_T\neq (1)}\leq K\exp\Sp{-\frac{T}{12H_{\textsf{G-BAI}}\Sp{\otheta_T}}},\\
            &\text{where }H_{\textsf{G-BAI}}\left(\otheta_T\right)=\frac{d}{\Delta_{(1)}^2}.
        \end{align*}
    }{
        $$\P_{\otheta_T}\Sp{J_T\neq (1)}\leq K\exp\Sp{-\frac{T}{12H_{\textsf{G-BAI}}\Sp{\otheta_T}}},\quad\text{where }H_{\textsf{G-BAI}}\left(\otheta_T\right)=\frac{d}{\Delta_{(1)}^2}.$$
    }
\end{restatable}

The proof of Theorem \ref{theo:adv_upper_bound} is deferred to Appendix \ref{sec:g_design_proof}. Here, we briefly compare this result with the one in multi-armed bandits, which can be treated as a special case of linear bandits by taking $\X=\Bp{\ve{e}_1, \dots, \ve{e}_K}$ to be the canonical vectors (standard basis) with $K=d$. 

In particular, \citet{abbasi2018best} shows that in multi-armed bandits setting, a simple uniform sampling algorithm reaches complexity $H_{\mathrm{UNIF}}\Sp{\otheta_T}=\frac{K}{\Delta_{(1)}^2}$ and it is optimal in non-stationary environments. Meanwhile, based on Theorem \ref{theo:adv_upper_bound}, we can see the complexity of \textsf{G-BAI} is $H_{\textsf{G-BAI}}\Sp{\otheta_T}=\frac{d}{\Delta_{(1)}^2}$, which is exactly $H_{\mathsf{UNIF}}(\otheta_T)$ if we treat multi-armed bandits as a special case of linear bandits since $d=K$. Furthermore, if we directly apply \textsf{G-BAI} to multi-armed bandits, meaning to use $\X=\Bp{\ve{e}_1, \dots, \ve{e}_K}$, then $\lambda^*$ is exactly the uniform distribution over $\X$. That is, in multi-armed bandits, \textsf{G-BAI} exactly recovers the optimal complexity in non-stationary environments, which shows that \textsf{G-BAI} is minimax optimal for linear bandits.


\section{A ROBUST ALGORITHM FOR STATIONARY/NON-STATIONARY ENVIRONMENTS}
\label{sec:bobw}


In this section, we present and analyze a new robust linear bandits BAI algorithm called \textsf{P1-RAGE}, which performs comparable to \textsf{G-BAI} in non-stationary environments but much better than it in stationary environments. We will show that it attains good error probability in both stationary and non-stationary environments simultaneously, without knowing a priori which environment it will encounter. We first discuss some intuitions behind the algorithm design.

\textbf{Stationary environments.} The development of our algorithm \textsf{P1-RAGE} is largely inspired by the high-level idea of the robust algorithm \textsf{P1}, proposed in \citet{abbasi2018best}, and the allocation strategy of \textsf{RAGE}, proposed in \citet{fiez2019sequential}. In particular, as discussed in \citet{abbasi2018best}, in multi-armed bandits, to minimize the error probability in stationary environment, we need to control the estimation variance of the optimal arm well enough. Therefore, at each time step, algorithm \textsf{P1} pulls the current estimated best arm with the highest probability (unnormalized ``probability one''), then subsequently the second best arm with second highest probability (unnormalized ``probability half'') and so on. 
We can notice that it actually matches the allocation strategy of the successive halving algorithm in \citet{karnin2013almost}, which is proved to be near-optimal for BAI in stationary multi-armed bandits. Therefore, we design our probability allocation based on the allocation strategy of \textsf{RAGE}, which is proven to be near-optimal for fixed-confidence BAI in stationary linear bandits \citep{fiez2019sequential}. In particular, with the estimated parameter $\htheta_t$, we first find the estimated best arms $\hat{x}^*_t=\argmax_{x\in\X}x^\top\htheta_t$. Then, we use a subroutine to repeatedly and virtually eliminate arms with estimated gaps larger than certain threshold and compute $\mc{XY}$-allocation of the (virtually) remaining arms.\footnote{The elimination is virtual because no samples are collected during the elimination subroutine.} Then, we average over the allocation probabilities computed during each iteration.


\textbf{Non-stationary environments.} Finally, to address the potential non-stationarity in environments, we uniformly mix the allocation probability computed above with a G-optimal design. With such a mixture, the variance over all arms can be controlled well and thus the algorithm will be robust for both stationary and non-stationary environments. The details of \textsf{P1-RAGE} are summarized in Algorithm \ref{algo:p1_rage} and the subroutine to compute the allocation probability, called \textsf{RAGE-Elimination}, is summarized in Algorithm \ref{algo:rage_elimination}.

\begin{algorithm}[ht]
    \caption{P1-RAGE}
    \label{algo:p1_rage}
    \begin{algorithmic}[1]
        \STATE \textbf{Input:} budget, $T\in\mathbb{N}$; arm set $\mc{X}\subset\R^d$; maximum number of virtual phases, $m$
        \STATE Compute G-optimal design $\lambda^*$ based on Eq. \eqref{equ:g_design} and initialize $\lambda_1=\lambda^*$
        \FOR{$t=1, 2, \dots, T$}
            \STATE Sample $x_t\sim\lambda_t$ and receive reward $r_t$
            \STATE Estimate $\widehat{\theta}_t\leftarrow\frac{1}{t}\sum_{s=1}^t\E_{x\sim\lambda_s}\Mp{xx^\top}^{-1}x_s r_s$
            \STATE Update $\lambda_{t+1}\leftarrow$\textsf{RAGE-Elimination}$(\htheta_t, m)$
            \IfTwoColumnElse{
                 \\ 
            }{}
            \hfill // \COMMENT{Call Algorithm \ref{algo:rage_elimination}}
           
        \ENDFOR
        \RETURN $\argmax_{x\in\X}x^\top\widehat{\theta}_T$
    \end{algorithmic}
\end{algorithm}

\begin{algorithm}[ht]
    \caption{RAGE-Elimination}
    \label{algo:rage_elimination}
    \begin{algorithmic}[1]
        \STATE \textbf{Input:} arm set $\mc{X}\subset\R^d$; current estimate $\htheta_t$; maximum number of virtual phases, $m$
        \STATE Find $\hat{x}^*_t\leftarrow\argmax_{x\in\X}x^\top\htheta_t$
        \STATE Initialize $\X_t^{(0)}\leftarrow\X$ and $i\leftarrow 0$
        \WHILE{$|\X_{t}^{(i)}|> 1$ and $i\leq m$}
            \STATE $\lambda^{(i)}_t\leftarrow\arginf_{\lambda\in\triangle_{\X}}\max_{x, x'\in\X_t^{(i)}}\Norm{x-x'}^2_{A(\lambda)^{-1}}$
            \STATE $\mc{X}_{t}^{(i+1)} \leftarrow \Bp{ x \in \mc{X}_t^{(i)}\left|\  \htheta_t^\top(\hat{x}^*_t-x) \leq 2^{-i}\right. }$
            \STATE $i\leftarrow i+1$
        \ENDWHILE
        \RETURN $(\bar{\lambda}_t+\lambda^*)/2$, where $\bar{\lambda}_t= \frac{1}{i}\sum_{i'=0}^{i-1}\lambda_t^{(i')}$
    \end{algorithmic}
\end{algorithm}

We bound the error probability of \textsf{P1-RAGE} under both stationary and non-stationary settings in the following theorem and its proof is deferred to Appendix \ref{sec:bobw_proof}.
\begin{theorem}[Error Probability of \textsf{P1-RAGE}]
    \label{theo:bobw_upper_bound}
    Fix arm set $\X\subset\R^d$ with $\abs{\X}=K$ and budget $T$. For a stationary environment with unknown parameter $\theta$, if $m\geq i_0=\ceil{\log_2\Sp{1/\Delta_{(1)}}}+1$, then there exists absolute constant $c>0$ such that the error probability of \textsf{P1-RAGE} satisfies
    \IfTwoColumnElse{
        \fontsize{9.5}{9.5}
        $$\P_{\theta}\Sp{J_T\neq (1)}\leq 2i_0 KT\exp\Sp{-\frac{cT}{H_{\textsf{P1-RAGE}}(\theta)}},$$
        \begin{equation}
            \label{equ:H_bobw}
            \begin{split}
                &H_{\textsf{P1-RAGE}}(\theta)= \frac{mi_0}{\Delta_{(1)}}\inf_{\lambda\in\triangle_{\X}}\max_{x\neq x_{(1)}}\frac{\Norm{x-x_{(1)}}^2_{A(\lambda)^{-1}}}{\Delta_x} \\
                &\quad + \frac{m\sqrt{d}}{\Delta_{(1)}}\inf_{\lambda\in\triangle_{\X}}\max_{x\neq x_{(1)}}\Norm{x-x_{(1)}}_{A(\lambda)^{-1}}.
            \end{split}
        \end{equation}
        \normalsize
    }{
        \begin{align}
            \P_{\theta}\Sp{J_T\neq (1)}\leq & 2i_0 KT\exp\Sp{-\frac{cT}{H_{\textsf{P1-RAGE}}(\theta)}},\nonumber\\
            \text{where } H_{\textsf{P1-RAGE}}(\theta)=& \frac{mi_0}{\Delta_{(1)}}\inf_{\lambda\in\triangle_{\X}}\max_{x\neq x_{(1)}}\frac{\Norm{x-x_{(1)}}^2_{A(\lambda)^{-1}}}{\Delta_x} + \frac{m\sqrt{d}}{\Delta_{(1)}}\inf_{\lambda\in\triangle_{\X}}\max_{x\neq x_{(1)}}\Norm{x-x_{(1)}}_{A(\lambda)^{-1}}.\label{equ:H_bobw}
        \end{align}
    }
    
    For a non-stationary environment with unknown parameter $\Bp{\theta_t}_{t=1}^{T}$, there exists absolute constant $c'>0$ such that the error probability of \textsf{P1-RAGE} satisfies
    $$\P_{\otheta_T}\Sp{J_T\neq (1)}\leq K\exp\Sp{-\frac{c'T\Delta_{(1)}^2}{d}}.$$
\end{theorem}

We can immediately see that in non-stationary environments, the error probability of \textsf{P1-RAGE} matches (up to a constant) with \textsf{G-BAI}, showing that \textsf{P1-RAGE} is minimax optimal for linear bandits under non-stationarity. On the other hand, because of the $\frac{1}{\Delta_{(1)}}$ factor, we can see that in stationary environments, $H_{\textsf{P1-RAGE}}(\theta)\gtrsim H_{\textsf{LB}}(\theta)$ (defined in Eq. \eqref{equ:rho_star}), which implies that \textsf{P1-RAGE} is suboptimal in stationary settings. However, this should be expected since even for multi-armed bandits, as proved in \citet{abbasi2018best}, it is impossible for an algorithm to achieve $H_{\textsf{LB}}(\theta)$ while being robust to non-stationarity, let alone linear bandits. 

Nevertheless, when applying Theorem \ref{theo:bobw_upper_bound} to multi-armed bandits ($\X=\Bp{\ve{e}_1, \dots, \ve{e}_K}$), as long as we choose $m\approx i_0$, we can show that (Corollary \ref{coro:bobw_linear_to_mab} in Appendix \ref{sec:bobw_proof})
$$H_{\textsf{P1-RAGE}}(\theta)=\widetilde{O}\Sp{\frac{1}{\Delta_{(1)}}\max_{k\in[K]}\frac{k}{\Delta_{(k)}}}=\widetilde{O}\Sp{H_{\mathrm{BOB}}(\theta)},$$
where $H_{\mathrm{BOB}}(\theta)$ is the best-of-both-worlds complexity proposed in \citet{abbasi2018best}. In particular, \citet{abbasi2018best} proves that $H_{\mathrm{BOB}}(\theta)$ is the best complexity that any algorithm can possibly achieve if it is constrained to be robust to non-stationarity. 
That is, again, our algorithm \textsf{P1-RAGE} retains the near-optimal complexity for stationary multi-armed bandits if it is constrained to be robust in non-stationary environments.

\begin{remark}
    Here, we do not elaborate the proof details of Theorem \ref{theo:bobw_upper_bound} mainly because we do not recognize them as widely applicable techniques. However, we do want to emphasize that this proof is by no means a simple extension of the analysis of the algorithm \textsf{P1} in \citet{abbasi2018best}. In particular, our proof uses a different set of virtual events based on the estimated gaps. Meanwhile, the analysis of subroutine \textsf{RAGE-Elimination} is intricately tailored to the unique characteristics of being a virtual elimination strategy, which is not presented in neither \textsf{RAGE} nor \textsf{P1} \citep{abbasi2018best, fiez2019sequential}.
\end{remark}

\textbf{Theoretical limitations of \textsf{P1-RAGE}.} Despite being near-optimal in multi-armed bandits, $H_{\textsf{P1-RAGE}}(\theta)$ includes an extra low-order term $\frac{m\sqrt{d}}{\Delta_{(1)}}\inf_{\lambda\in\triangle_{\X}}\max_{x\neq x_{(1)}}\Norm{x-x_{(1)}}_{A(\lambda)^{-1}}$. This term appears because the Bernstein's inequality requires a bound of the estimator's magnitude, which can be removed if the concentration bound only scales with the estimator's variance. Although this can often be accomplished by using Catoni's robust mean estimator \citep{wei2020taking}, it requires a concrete confidence level to be specified before estimation, which is not feasible in our fixed budget setting. Finding an approach to circumvent this difficulty and remove this extra term, or alternatively, demonstrate that it is necessary, is an open question.

\begin{remark}
    The question of whether the extra term is removable naturally relates the instance-dependent lower bound of this problem. However, proving an instance-dependent lower bound for our setting requires constructing both stationary and non-stationary counterexamples. This task is thereby more challenging compared to proving an instance-dependent lower bound for the fixed-budget best-arm identification problem in linear bandits within a purely stationary setting, an open question that persists (see \citet{yang2022minimax} for a minimax lower bound).  We thus leave establishing such instance-dependent lower bounds for future work.
\end{remark}

\textbf{Parameter choice of \textsf{P1-RAGE}.} Although \textsf{P1-RAGE} requires a user-specified parameter $m\geq \ceil{\log(1/\Delta_{(1)})}+1$ to bound the total number of virtual phases, it is not difficult to choose a reasonable value for this parameter in a practical implementation. On the one hand, since its dependence on $\Delta_{(1)}^{-1}$ is only logarithmic, taking some moderate value such as $m=25$ should safely satisfy $m\geq i_0$ for most practical scenarios; on the other hand, in most real-world applications, a sub-optimal arm should always be acceptable as long as its gap is small enough. Indeed, if we take $\epsilon$ to be the largest acceptable sub-optimality gap and take $m\geq \ceil{\log(1/\epsilon)}+1$, then \textsf{P1-RAGE} will output arm $x_{J_T}$ that satisfies $\Delta_{J_T}\leq \max\Bp{\epsilon, \Delta_{(1)}}$ with high probability in pure stationary environments (Corollary \ref{coro:epsilon_bai} in Appendix \ref{sec:bobw_proof}). That is, the output arm will either be an optimal arm if $\epsilon\leq\Delta_{(1)}$ or an arm with an acceptable suboptimality gap $\epsilon$ otherwise.

\section{EXPERIMENTS}
\label{sec:experiments}

In this section, we present our experiment results on several stationary/non-stationary environments. Since to the best of our knowledge, we are the first to propose best-arm identification algorithms that tackle non-stationarity in linear bandits, the algorithms from other works that we compare with are all specifically designed for stationary environments. In particular, we will compare our algorithms with \textsf{Peace}, which is the first fixed-budget algorithm for linear bandits and also inspires our algorithmic design \citep{katz2020empirical}, and \textsf{OD-LinBAI}, which is the most recent algorithm of this kind and is claimed to be minimax optimal \citep{yang2022minimax}.

Meanwhile, we also examine two additional heuristically designed algorithms for non-stationary environments. The first one is \textsf{P1-Peace}, which has the same design spirit as \textsf{P1-RAGE} but uses a different \textsf{Peace}-based virtual elimination subroutine; the second one is \textsf{Mixed-Peace}, which is a naive mixture of \textsf{Peace} and the G-optimal design. In particular, while \textsf{P1-RAGE/P1-Peace} combines G-optimal design with what \textsf{RAGE/Peace} would sample \emph{in a full run}, \textsf{Mixed-Peace} simply mixes G-optimal design with what \textsf{Peace} in a stationary environment samples \emph{at each time step}. The details of these two additional algorithms are summarized in Algorithm \ref{algo:p1_peace} and \ref{algo:mixed_peace} in Appendix \ref{sec:modified_algo}, respectively. More implementation details and additional experiments can be found in Appendix \ref{sec:experiment_details}.\footnote{Code repository is available at \url{https://github.com/FFTypeZero/bobw_linear}.}

\IfTwoColumnElse{
    \begin{figure}[ht]
        \centering
        \includegraphics[width=\linewidth]{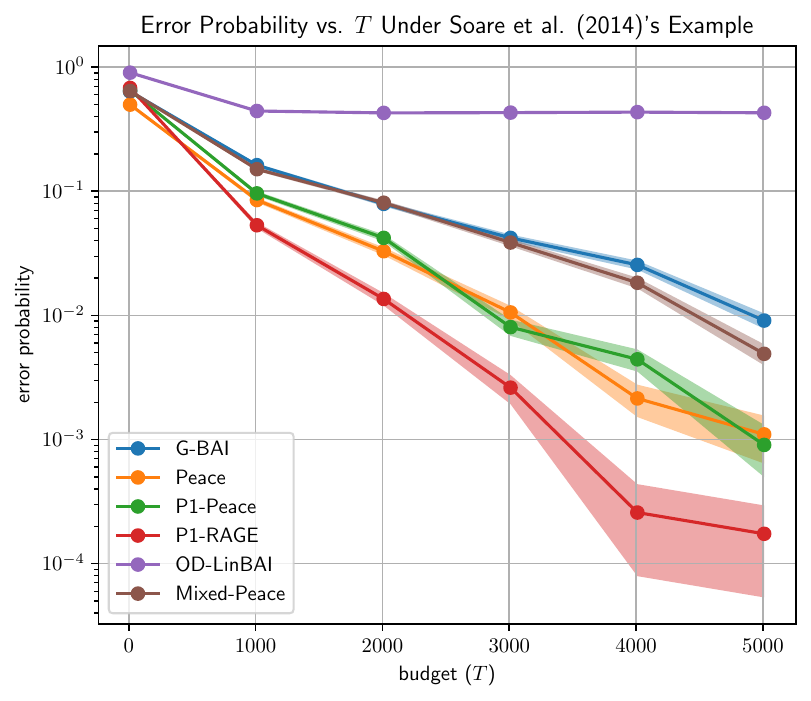}
        \caption{Each error probability is estimated through at least $2\times 10^4$ independent trials. The vertical axis is on log scale and the shaded area represents the $95\%$ confidence interval.}
        \label{fig:sto_soare}
    \end{figure}
}{
    \begin{figure}[ht]
        \centering
         \begin{subfigure}[b]{0.48\textwidth}
             \centering
             \includegraphics[width=\textwidth]{figs/sto_soare_log.pdf}
             \caption{Each error probability is estimated through at least $2\times 10^4$ independent trials.}
             \label{fig:sto_soare}
         \end{subfigure}
         \hfill
         \begin{subfigure}[b]{0.48\textwidth}
             \centering
             \includegraphics[width=\textwidth]{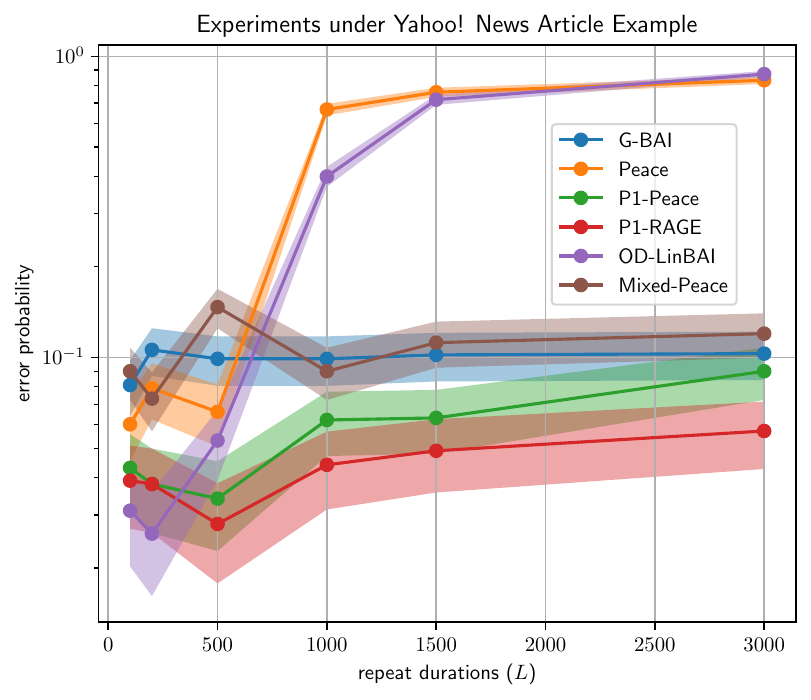}
             \caption{Each error probability is estimated through $1000$ independent trials.}
             \label{fig:adv_yahoo}
         \end{subfigure}
         \caption{The vertical axis is on log scale and the shaded area represents the $95\%$ confidence interval.}
    \end{figure}
}

\textbf{Stationary benchmark example.} First, as a sanity check, we consider the famous stationary benchmark example proposed in \citet{soare2014best}. In particular, we have $\X=\Bp{\ve{e}_1, \dots, \ve{e}_d, x'}$, where $x'=\cos(\omega)\ve{e}_1+\sin(\omega)\ve{e}_2$ with some small $\omega>0$, and $\otheta_T=\theta^*=2\ve{e}_1$ so that $x_{(1)}=\ve{e}_1$. An efficient algorithm should pick $\ve{e}_2$ frequently to reduce the variance in the direction of $\ve{e}_1-x'$. In this example, we pick $d=10$ and $\omega=0.1$. 

The results are shown in Figure \ref{fig:sto_soare}. We can see that both our algorithms, \textsf{P1-RAGE} and \textsf{P1-Peace}, perform better than \textsf{G-BAI} and comparably with \textsf{Peace}, showing that our algorithms maintain good performance in stationary environments. Meanwhile, we also notice that \textsf{Mixed-Peace} has performance only comparable to \textsf{G-BAI}, showing that naively mixing the allocation strategy with the G-optimal design can downgrade the performance in stationary environments.

\begin{figure*}[ht]
    \centering
    \includegraphics[width=\linewidth]{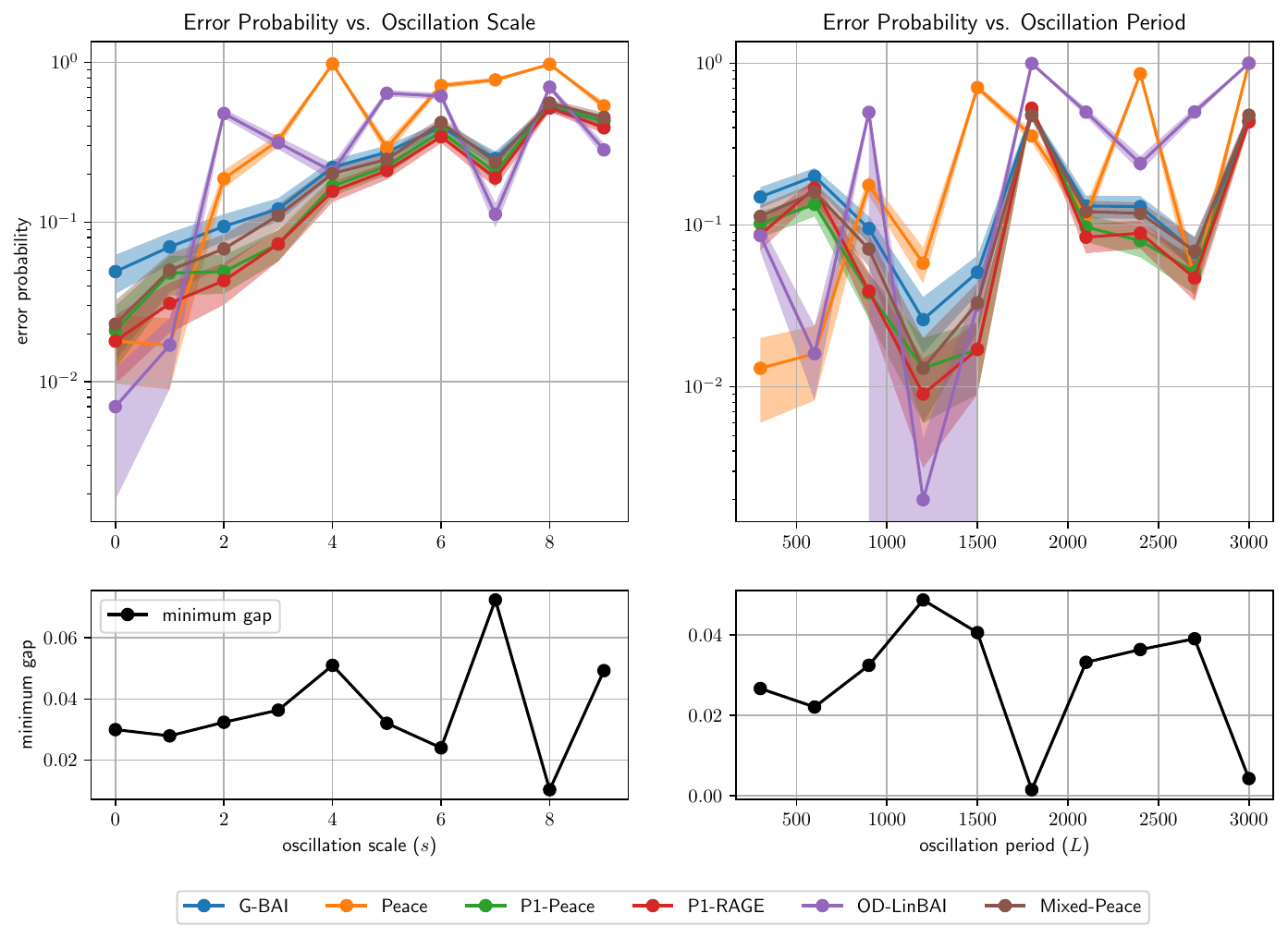}
    \caption{Each error probability is estimated through 1000 repeated trials. The bottom two plots give the minimum gap $\Delta_{(1)}$ of each instance as a function of oscillation scale $s$ and oscillation period $L$.}
    \label{fig:adv_multi}
\end{figure*}

\textbf{Non-stationary multivariate testing example.} We consider a multivariate testing example from \citet{fiez2019sequential}, which is also similar to the one discussed in Introduction. Considering a webpage with $D$ distinct slots and suppose each slot has two content choices, where we represent each layout as an element $w\in\mc{W}=\Bp{-1, 1}^D$. We hope to maximize the click-through rate and we assume it linearly depends on a layout-determined arm $x\in \X$ in a form of
$$x^\top\theta^*=\theta_0^*+\alpha_1\sum_{j=1}^{D}\theta_j^*w_j+\alpha_2\sum_{k=1}^{D-1}\sum_{\ell=k+1}^{D}\theta^*_{k, \ell}w_kw_{\ell}.$$
Here $\theta_0^*$ is the common bias, $\theta_j^*$ is the weight of $j$-th slot and $\theta^*_{k, \ell}$ is the weight of the interaction between $k$-th and $\ell$-th slots. Because of the periodic nature of people's life cycle, it is very likely that the real-world weights will periodically change. Therefore, to construct a non-stationary environment, we randomly oscillate the weights with scale $s$ and period $L$ to get
\IfTwoColumnElse{
    \begin{align*}
        &\theta_{t, i}=\theta^*_i+sI\Norm{\theta^*}_{\infty}\sin\Sp{\frac{2\pi t}{L}+\phi_i},\\
        &\text{where } I\sim\mathsf{Unif}(\Bp{0, 1}), \phi_i\sim\mathsf{Unif}([0, 2\pi]).
    \end{align*}
}{
    $$\theta_{t, i}=\theta^*_i+sI\Norm{\theta^*}_{\infty}\sin\Sp{\frac{2\pi t}{L}+\phi_i}, \quad\text{where } I\sim\mathsf{Unif}(\Bp{0, 1}), \phi_i\sim\mathsf{Unif}([0, 2\pi]).$$
}
Here, in the first series of instances, we fix $L=900$ and take values $s\in\Bp{0, 1, \dots, 9}$, and in the second series of instances, we fix $s=2$ and take values $L\in\Bp{300, 600, \dots, 3000}$. Finally, we take $\alpha_1=1$, $\alpha_2=0.5$, sample each component of $\theta^*$ uniformly in $[-0.1, 0.1]$ and guarantee that $\otheta_T$ has the same optimal arm as $\theta^*$. We take $T=10^4$ for all settings and the results are shown in Figure \ref{fig:adv_multi}.

 From the plots, we can see that the error probabilities of \textsf{Peace} and \textsf{OD-LinBAI}, algorithms designed for stationary environments, can range from near 0 to 1 in different non-stationary environments, which is quite unstable. Meanwhile, we can see that the performance of the other four algorithms, which all in certain way contain a G-optimal design, is relatively much more stable.\footnote{All algorithms fluctuate in the upper right plot mainly because the minimum gaps also have large fluctuation.} Furthermore, among these four algorithms, we can see that our algorithms \textsf{P1-RAGE} and \textsf{P1-Peace} consistently outperform (never worse than) \textsf{G-BAI} and \textsf{Mixed-Peace}.


\textbf{Non-stationary click-through example.} To create an instance using real-world data, we use the Yahoo! Webscope Dataset R6A \citep{r6adataset}.\footnote{\url{https://webscope.sandbox.yahoo.com/}} This dataset contains a fraction of user click log of Yahoo!'s news article from May 1st, 2009 to May 10th, 2009. For each click, we take the outer product between user and article features to get a vector in $\R^{36}$ and then we run a principle component analysis to get arm set $\mc{Z}\subset\R^{24}$. To create a non-stationary example, we take data from May 1st to May 7th and for each day's data, we fit a ridge regression with regularization $0.01$, obtaining $\theta^*_1, \dots, \theta^*_7$, which can be used to simulate user's weekly periodic behavior. Suppose we receive $L$ visits each day, then, we can define a non-stationary environment where each period consists of $\theta_1^*,\dots, \theta_1^*, \dots, \theta_7^*, \dots, \theta_7^*$ and each $\theta_i^*$ repeats for $L$ times. Finally, we form our arm set $\X$ by picking the optimal arm from $\mc{Z}$ plus 23 randomly picked arms with gap at least $0.05$ so that $\mathrm{span}(\X)=\R^{24}$. We take $T=2.1\times 10^4$ and the results are shown in Figure \ref{fig:adv_yahoo}. Again, we can see that the performance of \textsf{Peace} and \textsf{OD-LinBAI} is very unstable and the performance of \textsf{P1-RAGE} and \textsf{P1-Peace} consistently outperforms the other two naive G-optimal-design-based algorithms, \textsf{G-BAI} and \textsf{Mixed-Peace}.

\IfTwoColumnElse{
    \begin{figure}[ht]
        \centering
        \includegraphics[width=\linewidth]{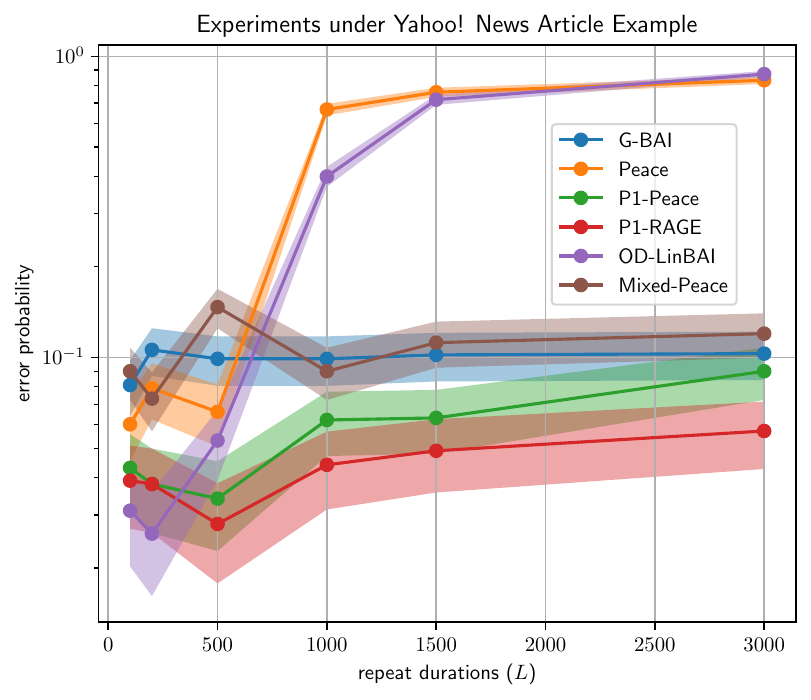}
        \caption{Each error probability is estimated through $1000$ independent trials. The vertical axis is on log scale and the shaded area represents the $95\%$ confidence interval.}
        \label{fig:adv_yahoo}
    \end{figure}
}{

}


\section{CONCLUSIONS AND FUTURE WORK}
\label{sec:conclusion}

To the best of our knowledge, in this paper, we present the first two novel robust linear bandits algorithm for fixed-budget best-arm identification, \textsf{P1-RAGE} and \textsf{P1-Peace}, that tackle stationary and non-stationary environments simultaneously while being agnostic to the environment. Theoretically, we prove error probability bounds of \textsf{P1-RAGE} in both stationary and non-stationary environments. Empirically, we show that in stationary settings, both \textsf{P1-RAGE} and \textsf{P1-Peace} perform comparably with algorithms designed for such environments, and in non-stationary settings, they consistently outperform naive algorithms based on G-optimal design.

Finally, several questions still remain open. Is the extra term in $H_{\textsf{P1-RAGE}}(\theta)$, as discussed in Section \ref{sec:bobw}, necessary? What is the optimal complexity for this mixed stationary/non-stationary settings? Answering these questions can serve as promising future directions.

\section*{Acknowledgements}

ZX sincerely thanks Anze Wang for his great favor offered during the time of paper writing. This work was supported in part by the NSF TRIPODS II grant DMS 2023166, NSF CCF 2007036, NSF CCF 2212261 and Microsoft Grant for Customer Experience Innovation.

\bibliography{References}
\bibliographystyle{plainnat}

\newpage
\appendix

\tableofcontents
\newpage

\section{ADDITIONAL ALGORITHMS IN IMPLEMENTATION}

\subsection{A \textsf{Peace}-based Robust Algorithm}
\label{sec:modified_algo}

In this section, we briefly explain how we design \textsf{P1-Peace} based on intuition similar to \textsf{P1-RAGE} and make it computationally efficient. First, we propose another subroutine, called \textsf{Peace-Elimination}, based on the elimination strategy in \textsf{Peace} \citet{katz2020empirical}, which has the same spirit as \textsf{RAGE}. Similar to \textsf{RAGE-Elimination}, \textsf{Peace-Elimination} also repeatedly computes $\mc{XY}$-allocation, but (virtually) eliminate arms so that the value of the remaining arms' optimal $\mc{XY}$-design is halved. In addition, in \textsf{P1-Peace}, we only update the sampling distribution $\lambda_t$ after a period of time. The intuition is that if the environment is stationary, then we do not need to update our allocation probability frequently just like \textsf{RAGE} and \textsf{Peace}; if the environment is non-stationary, then the non-stationarity is handled by the mixed G-optimal design $\lambda^*$, which is fixed from the very beginning. Therefore, updating $\lambda_t$ in a low frequency should not severely harm the performance. The new algorithm and elimination subroutine are summarized in Algorithm \ref{algo:p1_peace} and \ref{algo:peace_elimination}.

For convenience of presentation, for arm set $\mc{Z}\subset\R^d$ and distribution $\lambda\in\triangle_{\X}$, we define
\begin{equation}
    \label{equ:rho_z}
    \rho(\mc{Z}, \lambda)=\max_{x, x'\in\mc{Z}}\Norm{x-x'}^2_{A(\lambda)^{-1}}.
\end{equation}

\begin{algorithm}[ht]
    \caption{P1-Peace}
    \label{algo:p1_peace}
    \begin{algorithmic}[1]
        \STATE \textbf{Input:} budget, $T\in\mathbb{N}$; arm set $\mc{X}\subset\R^d$
        \STATE Compute epoch length $R\leftarrow\floor{\frac{T}{\log_2(\inf_{\lambda\in\triangle_{\X}}\rho(\X, \lambda))}}$
        \STATE Compute G-optimal design $\lambda^*$ based on equation \eqref{equ:g_design} and initialize $\lambda_1=\lambda^*$
        \FOR{$t=1, 2, \dots, T$}
            \STATE Sample $x_t\sim\lambda_t$ and receive reward $r_t$
            \STATE Estimate $\widehat{\theta}_t\leftarrow\frac{1}{t}\sum_{s=1}^t\E_{x\sim\lambda_s}\Mp{xx^\top}^{-1}x_s r_s$
            \STATE $\lambda_{t+1}\leftarrow\lambda_t$
            \IF{$t-1=cR$ for some integer $c$}
                \STATE Update $\lambda_{t+1}\leftarrow$\textsf{Peace-Elimination}$(\htheta_t)$
            \ENDIF
        \ENDFOR
        \RETURN $\argmax_{x\in\X}x^\top\widehat{\theta}_T$
    \end{algorithmic}
\end{algorithm}

\begin{algorithm}
    \caption{Peace-Elimination}
    \label{algo:peace_elimination}
    \begin{algorithmic}[1]
        \STATE \textbf{Input:} arm set $\mc{X}\subset\R^d$; current estimate $\htheta_t$
        \STATE Find index $\widehat{(k)}_t$ such that $x_{\widehat{(1)}_t}^\top\htheta_t\geq x_{\widehat{(2)}_t}^\top\htheta_t\geq\dots\geq x_{\widehat{(K)}_t}^\top\htheta_t$
        \STATE Initialize $\X_t^{(0)}\leftarrow\X$ and $i\leftarrow 0$
        \WHILE{$|\X_{t}^{(i)}|> 1$}
            \STATE Compute $\lambda^{(i)}_t\leftarrow\arginf_{\lambda\in\triangle_{\X}}\rho(\X_t^{(i)}, \lambda)$
            \STATE Find the largest index $k_i$ such that
            $$\inf_{\lambda\in\triangle_{\X}}\rho\Sp{\{x_{\widehat{(1)}_t}, \dots, x_{\widehat{(k_i)}_t}\}}\leq \frac{1}{2}\cdot\inf_{\lambda\in\triangle_{\X}}\rho(\X_t^{(i)}, \lambda)$$
            \STATE Update $\X_{t}^{(i+1)}\leftarrow\Bp{x_{\widehat{(1)}_t}, \dots, x_{\widehat{(k_i)}_t}}$
            \STATE $i\leftarrow i+1$
        \ENDWHILE
        \RETURN $(\bar{\lambda}_t+\lambda^*)/2$, where $\bar{\lambda}_t= \frac{1}{i}\sum_{i'=0}^{i-1}\lambda_t^{(i')}$
    \end{algorithmic}
\end{algorithm}

\subsection{A Naive Baseline Mixed Algorithm}
\label{sec:naive_algo}
In this section, we present a naive mixture of \textsf{Peace} and the G-optimal design, called \textsf{Mixed-Peace}, which eliminates arms and computes design $\lambda_k$ during each epoch exactly the same as \textsf{Peace}. The only differences are that \textsf{Mixed-Peace} uses IPS estimator and when pulling an arm, it will pull an arm by following $x_t\sim(\lambda_k+\lambda^*)/2$, where $\lambda^*$ is the G-optimal design defined in equation \eqref{equ:g_design}. Its details are summarized in Algorithm \ref{algo:mixed_peace}.


\begin{algorithm}
    \caption{Mixed-Peace}
    \label{algo:mixed_peace}
    \begin{algorithmic}[1]
        \STATE \textbf{Input:} budget, $T\in\mathbb{N}$; arm set $\mc{X}\subset\R^d$
        \STATE Initialize $R\leftarrow\ceil{\log_2\Sp{\inf_{\lambda\in\triangle_{\X}}\rho(\X, \lambda)}}$, $N\leftarrow\floor{\frac{T}{R}}$, $\mc{X}_0\leftarrow\X$, $\htheta_0\leftarrow\ve{0}$ and $t\leftarrow 1$
        \STATE Compute G-optimal design $\lambda^*$ using equation \eqref{equ:g_design}
        \FOR{$r=0, \dots, R$}
            \STATE Find $\lambda_r\leftarrow(\arginf_{\lambda\in\triangle_{\X}}\rho(\X_r, \lambda)+\lambda^*)/2$
            \WHILE{$t\leq\min\Bp{T, (r+1)N}$}
                \STATE Sample $x_t\sim \lambda_r$ and receive reward $r_t$
                \STATE Estimate $\widehat{\theta}_t\leftarrow \frac{t-1}{t}\cdot \htheta_{t-1}+\frac{1}{t}\cdot\E_{x\sim\lambda_r}\Mp{xx^\top}^{-1}x_t r_t$
                \STATE $t\leftarrow t+1$
            \ENDWHILE
            \IF{$\abs{\X_r}>1$}
                \STATE Reindex $\X_r$ such that $x_1^\top\htheta_t\geq x_2^\top\htheta_t\geq\dots\geq x_{n_r}^\top\htheta_t$, where $n_r=\abs{\X_r}$
                \STATE Find the largest index $k_r$ such that 
                $$\inf_{\lambda\in\triangle_{\X}}\rho(\Bp{x_1, \dots, x_{k_r}}, \lambda)\leq \frac{1}{2}\cdot\inf_{\lambda\in\triangle_{\X}}\rho(\X_r, \lambda)$$
                \STATE Update $\X_{r+1}\leftarrow\Bp{x_1, \dots, x_{k_r}}$
            \ENDIF
        \ENDFOR
        \textbf{return} $\argmax_{x\in\X}x^\top\htheta_T$
    \end{algorithmic}
\end{algorithm}
\section{ERROR PROBABILITY OF ALGORITHM \ref{algo:gbai} IN NON-STATIONARY ENVIRONMENTS}
\label{sec:g_design_proof}

\advupperbound*

\begin{proof}
    Based on the recommendation rule $x_{J_T}=\argmax_{x\in\X}x^\top\htheta_T$, we have
    \begin{align}
        \P\Sp{J_T\neq (1)}=&\P\Sp{\exists k \in [2:K] \text{ s.t. } x_{(k)}^\top\htheta_T\geq x_{(1)}^\top\htheta_T}\nonumber\\
        \leq & \P\Sp{\exists k\in[2:K]\text{ s.t. } x_{(k)}^\top\htheta_T-x_{(k)}^\top\otheta_T\geq\frac{\Delta_{(k)}}{2} \text{ or } x_{(1)}^\top\htheta_T-x_{(1)}^\top\otheta_T\leq -\frac{\Delta_{(1)}}{2}}\nonumber\\
        \leq & \P\Sp{x_{(1)}^\top\htheta_T-x_{(1)}^\top\otheta_T\leq -\frac{\Delta_{(1)}}{2}} + \sum_{k=2}^{K}\P\Sp{x_{(k)}^\top\htheta_T-x_{(k)}^\top\otheta_T\geq\frac{\Delta_{(k)}}{2}}.\label{equ:g_bernstein}
    \end{align}
    The above terms can be bounded by Bernstein's inequality. In particular, for the first term, we have
    $$\P\Sp{x_{(1)}^\top\htheta_T-x_{(1)}^\top\otheta_T\leq -\frac{\Delta_{(1)}}{2}}=\P\Sp{\sum_{t=1}^{T}x_{(1)}^\top\Sp{A(\lambda^*)^{-1}x_tr_t-\theta_t}\leq-\frac{T\Delta_{(1)}}{2}}.$$
    Since IPS estimator is unbiased, $x_{(1)}^\top\Sp{A(\lambda^*)^{-1}x_tr_t-\theta_t}$ is a zero-mean random variable. Based on our bounded reward assumption, we have
    $$\abs{x_{(1)}^\top\Sp{A(\lambda^*)^{-1}x_tr_t-\theta_t}}\leq \abs{x_{(1)}^\top A(\lambda^*)^{-1}x_t} + 2\leq \Norm{x_{(1)}}_{A(\lambda^*)^{-1}}\Norm{x_t}_{A(\lambda^*)^{-1}}+2\leq d+2\leq 3d,$$
    where we use the property of G-optimal design $\max_{x\in\X}\Norm{x}^2_{A(\lambda^*)^{-1}}\leq d$. We can similarly bound its variance by 
    \begin{align*}
        \E\Mp{\Sp{x_{(1)}^\top\Sp{A(\lambda^*)^{-1}x_tr_t-\theta_t}}^2}\leq & \E\Mp{\Sp{x_{(1)}^\top A(\lambda^*)^{-1}x_t}^2}\\
        = & x_{(1)}^\top A(\lambda^*)^{-1}\E\Mp{x_tx_t^\top}A(\lambda^*)^{-1}x_{(1)}\\
        = & x_{(1)}^\top A(\lambda^*)^{-1}A(\lambda^*)A(\lambda^*)^{-1}x_{(1)}\tag{Since $x_t\sim\lambda^*$ by algorithm}\\
        = & \Norm{x_{(1)}}^2_{A(\lambda^*)^{-1}} \leq  d
    \end{align*}
    Thus, by Bernstein's inequality, we have
    $$\P\Sp{x_{(1)}^\top\htheta_T-x_{(1)}^\top\otheta_T\leq -\frac{\Delta_{(1)}}{2}}\leq \exp\Sp{-\frac{T^2\Delta^2_{(1)}/8}{Td+Td\Delta_{(1)}/2}}\leq\exp\Sp{-\frac{T\Delta_{(1)}^2}{12d}},$$
    where the last inequality uses the assumption that $\Delta_{(1)}\leq 1$. By similarly applying Bernstein's inequality to other terms in \eqref{equ:g_bernstein}, we can then have
    \begin{align*}
        \P\Sp{J_T\neq x_{(1)}} \leq &\P\Sp{x_{(1)}^\top\htheta_T-x_{(1)}^\top\otheta_T\leq -\frac{\Delta_{(1)}}{2}} + \sum_{k=2}^{K}\P\Sp{x_{(k)}^\top\htheta_T-x_{(k)}^\top\otheta_T\geq\frac{\Delta_{(k)}}{2}}\\
        \leq & \sum_{k=1}^{K}\exp\Sp{-\frac{T\Delta_{(k)}^2}{12d}}\\
        \leq & K\exp\Sp{-\frac{T\Delta_{(1)}^2}{12d}}.
    \end{align*}
\end{proof}
\section{ERROR PROBABILITY OF ALGORITHM \ref{algo:p1_rage}}
\label{sec:bobw_proof}
\subsection{Stationary Environments}
We first prove an error probability of Algorithm \ref{algo:p1_rage} in stationary environments that contains unspecified parameters from the virtual phases. Without loss of generality, assume that the arms $x_1, \ldots, x_K$ are ordered such that $\theta^\top x_1 > \theta^\top x_2 \geq \dots \geq \theta^\top x_K$ and $\Delta_1=\Delta_2\leq\Delta_3\leq\dots\leq\Delta_K$.

Throughout this section, we will the following definitions: $i_0 = \ceil{\log_2(1/\Delta_{1})} + 1$, $\mc{A}_i=\Bp{x\in\X\mid \Delta_x\leq 2\cdot 2^{-i}}$, $\bar{i}(k)=\max\Bp{i\in[i_0 - 1]\mid \Delta_k\leq 2^{-i}}$ and  
    $$f(\mc{A}_i)=\min_{\lambda\in\triangle_{\X}}\max_{x, x'\in\mc{A}_i}\Norm{x-x'}^2_{A(\lambda)^{-1}}.$$

\begin{theorem}
    \label{theo:bobw_error_prob_raw}
    Let $\mc{D}=\Bp{\ve{a}\in[0, 1]^{i_0+1}\mid 0=a_0<a_1\leq a_2 \leq \ldots \leq a_{i_0}=1}$. Then, if $m\geq i_0$, The error probability of Algorithm \ref{algo:p1_rage} in a stationary environment with parameter $\theta$ is bounded as 
    \begin{align}
        \P_{\theta}\Sp{J_T\neq 1} \leq& 2i_0 KT \exp\Sp{-\frac{T}{\overline{H}_{\textsf{P1-RAGE}}(\theta)}},\nonumber\\
        \overline{H}_{\textsf{P1-RAGE}}(\theta)=&\min_{\ve{a}\in\mc{D}} \max_{k\in[K]}\frac{48m\sum_{i'=1}^{\bar{i}(k)}(a_{i'}-a_{i'-1})f(\mc{A}_{i'-2}) + 8(m\sqrt{df(\mc{X})}+1)a_{\bar{i}(k)}\Delta_k}{3a_{\bar{i}(k)}^2\Delta_k^2}.\label{equ:H_p1rage}
    \end{align}
\end{theorem}

\begin{proof}
With $0=n_0<n_1\leq n_2 \leq\ldots \leq n_{i_0}=T$.\footnote{We do not specify the values of $n_1, \dots, n_{i_0-1}$ for now.} we define the event $\xi_i$ with $i \geq 1$ as follows: after $n_i$ samples all the arms with true gap smaller than $2\cdot2^{-i}$ are estimated with precision $2^{-i}/2$, which is
\begin{align*}
    \xi_i = \{\forall t \geq n_i, \forall k\in[K] \text{ s.t. } \Delta_k \leq 2 \cdot 2^{-i} \implies | \Delta_k - \widehat{\Delta}^{(t)}_k| < 2^{-i}/2\} ,
\end{align*}
where $\widehat{\Delta}^{(t)}_k = (x_1-x_k)^\top\widehat{\theta}^{(t)}$ for $k>1$ and $\widehat{\Delta}^{(t)}_1 = (x_1-x_2)^\top\widehat{\theta}^{(t)}$.
We first show how these events $\Bp{\xi_i}_{i=1}^{i_0}$ relate the correctness of Algorithm \ref{algo:p1_rage}.

\textbf{Correctness.} If $\bigcap_{i=1}^{i_0} \xi_i$ holds then the algorithm successfully identifies the best arm. Indeed, if we assume it does not, then there must exist non-optimal arm $k_0$ such that $\widehat{\Delta}^{(T)}_{k_0} < 0$. As $\bigcap_{i=1}^{i_0} \xi_i$ holds, for some $i'\leq i_0$, it holds that $2^{-i'} < \Delta_{k_0} \leq 2\cdot2^{-i'}$ and then $ | \Delta_{k_0} - \widehat{\Delta}^{(T)}_{k_0}| < 2^{-i'}/2$. Therefore, we have $2^{-i'} < \Delta_{k_0} \leq \Delta_{k_0} - \widehat{\Delta}^{(T)}_{k_0} \leq | \Delta_{k_0} - \widehat{\Delta}^{(T)}_{k_0}| \leq 2^{-i'}/2$, which is a contradiction.

Thus, the error probability is upper bounded by $\P\left(\bigcup_{i=1}^{i_0} \xi_i^c \right)$, which gives us

\begin{align*}
    \P\Sp{J_T\neq 1} \leq& \P\Sp{\bigcup_{i=1}^{i_0} \xi_i^c} = \P\Sp{\bigcup_{i=1}^{i_0}\Sp{\xi_i^c\setminus\bigcup_{j=1}^{i-1}\xi_j^c}} \leq  \sum_{i=1}^{i_0}\P\Sp{\xi_i^c\setminus \bigcup_{j=1}^{i-1}\xi_j^c} \\
    =& \sum_{i=1}^{i_0}\P\Sp{\xi_i^c\cap\Sp{\bigcup_{j=1}^{i-1}\xi_j^c}^c} = \sum_{i=1}^{i_0}\P\Sp{\xi_i^c\cap\Sp{\bigcap_{j=1}^{i-1}\xi_j}} \\
    \leq& \sum_{i=1}^{i_0}\P\Sp{ \xi_i^c\left| \bigcap_{j=1}^{i-1} \xi_j\right.}.
\end{align*}

\textbf{Bernstein's inequality.} Now, we just need to find an upper bound of $\P\left(\xi_i^c \left|\bigcap_{j=1}^{i-1} \xi_{j}\right.\right)$. Assume $\exists t \geq n_i, \exists k\in[K] \text{ s.t. } \Delta_k \leq 2 \cdot 2^{-i}$.\footnote{Otherwise, $\xi_i$ is vacuously true and $\P\Sp{\xi_i^c}=0$.} Then, we have
\begin{align}
    &\P(| \Delta_k - \widehat{\Delta}^{(t)}_k| \geq 2^{-i}/2)\nonumber\\
    =& \P( |(\theta - \widehat{\theta}_t)^\top (x_1-x_k)| \geq 2^{-i}/2 ) \label{equ:epsilon_val}\\
    =&\P\Sp{\abs{\sum_{s=1}^{t}\left(\theta-A(\lambda_s)^{-1}x_sr_s\right)^\top (x_1-x_k)}\geq 2^{-i}t/2}\nonumber\\
    \overset{\text{(a)}}{\leq}& 2\exp\Sp{-\frac{2^{-2i}t^2/8}{2\sum_{s=1}^{t}\Norm{x_1-x_k}^2_{A(\bar{\lambda}_s)^{-1}}+\Sp{\sqrt{d}\max_{s\in[1:t]}\Norm{x_1-x_k}_{A(\bar{\lambda}_s)^{-1}}+1}t2^{-i}/3}}\tag{By Bernstein's inequality for martingale differences \cite{freedman1975tail}}\\
    \leq& 2\exp\Sp{-\frac{2^{-2i}t^2/8}{\text{term I}}},\nonumber
\end{align}
\begin{align*}
    \text{where term I}=&2\sum_{i'=1}^{i}\sum_{s=n_{i'-1}+1}^{n_{i'}}\Norm{x_1-x_k}^2_{A(\bar{\lambda}_s)^{-1}}+2\sum_{s=n_i+1}^{t}\Norm{x_1-x_k}^2_{A(\bar{\lambda}_s)^{-1}}\\
    &\quad +\left(\sqrt{d}\max_{s\in[1:t]}\Norm{x_1-x_k}_{A(\bar{\lambda}_s)^{-1}}+1\right) \cdot\frac{t2^{-i}}{3}.
\end{align*}
Here, to use Bernstein's inequality for martingale differences in the inequality (a) above, we need to bound the variance and magnitude of $\left(\theta-A(\lambda_s)^{-1}x_sr_s\right)^\top (x_1-x_k)$ condition on $\lambda_s$.\footnote{Since IPS estimator is unbiased and $\lambda_s$ is determined by the history prior to time $s$, we have $\E\Mp{\left(\theta-A(\lambda_s)^{-1}x_sr_s\right)^\top (x_1-x_k)\mid \mc{H}_{s-1}}=0$, which implies that it is a martingale difference sequence.} In particular, we have
\begin{align*}
    \abs{\left(\theta-A(\lambda_s)^{-1}x_sr_s\right)^\top (x_1-x_k)} \leq& \abs{(x_1-x_k)^\top A(\lambda_s)^{-1}x_s} + \Delta_k\\
    \leq& \Norm{x_1-x_k}_{A(\lambda_s)^{-1}}\Norm{x_s}_{A(\lambda_s)^{-1}} + 2\\
    \leq& 2\sqrt{d}\Norm{x_1-x_k}_{A(\bar{\lambda}_s)^{-1}}+2.\tag{Since $\lambda_s=(\bar{\lambda}_s+\lambda^*)/2$ and $\lambda\mapsto\Norm{x_1-x_k}^2_{A(\lambda)^{-1}}$ is convex in $\lambda$}
\end{align*}
\begin{align*}
    &\E\Mp{\Sp{\left(\theta-A(\lambda_s)^{-1}x_sr_s\right)^\top (x_1-x_k)}^2\mid\lambda_s}\\
    \leq & \E\Mp{\Sp{(x_1-x_k)^\top A(\lambda_s)^{-1}x_s}^2\mid \lambda_s}\\
    =&(x_1-x_k)^\top A(\lambda_s)^{-1}\E\Mp{x_sx_s^\top\mid\lambda_s}A(\lambda_s)^{-1}(x_1-x_k)\\
    =&\Norm{x_1-x_k}^2_{A(\lambda_s)^{-1}}\\
    \leq& 2\Norm{x_1-x_k}^2_{A(\bar{\lambda}_s)^{-1}}.\tag{Since $\lambda_s=(\bar{\lambda}_s+\lambda^*)/2$}
\end{align*}

\textbf{Single-term error probability.} Now, we need to use the property of the subroutine \textsf{RAGE-Elimination} (Line \ref{algo:rage_elimination} of Algorithm \ref{algo:p1_rage}) that generates $\lambda_s$. That is, by Lemma~\ref{lmm:subroutine_guarantees}, since $x_k\in\mc{A}_i\subseteq\mc{A}_{i'}$ for $i'\leq i$ and $m\geq i_0$, for $s\in[n_{i'-1}+1, n_{i'}]$, we have $\Norm{x_1 - x_k}_{A(\bar{\lambda}_s)^{-1}}^2 \leq m\inf_{\lambda\in\triangle_{\X}} \max_{x, x'\in\mc{A}_{i'-2}}\Norm{x - x'}^2_{A(\lambda)^{-1}}\overset{\mathrm{def}}{=}m f(\mc{A}_{i'-2})$. Thus, we have
\begin{align*}
    &\P(| \Delta_k - \widehat{\Delta}^{(t)}_k| \geq 2^{-i}/2)\\
    \leq& 2\exp\Sp{-\frac{2^{-2i}t^2/8}{2m\sum_{i'=1}^{i}(n_{i'}-n_{i'-1})f(\mc{A}_{i'-2})+2m(t-n_i)f(\mc{A}_{i-1})+(m\sqrt{df(\mc{X})}+1)t2^{-i}/3}}\\
    \leq & 2\exp\Sp{-\frac{2^{-2i}n_i^2/8}{2m\sum_{i'=1}^{i}(n_{i'}-n_{i'-1})f(\mc{A}_{i'-2})+(m\sqrt{df(\mc{X})}+1)n_i2^{-i}/3}},
\end{align*}
where the last inequality above holds because of $t\geq n_i$ and a simple fact that $t\mapsto\frac{t^2}{at+b}$ is an increasing function when $t\geq 0$ if $a>0$ and $b>0$. 

\textbf{Final error probability.} Then, with the union bound over all $ t \geq n_i$ and $ k\in[K]$, it holds for any $0<n_1 \leq n_2 \ldots \leq n_i \leq T$ that
\begin{align*}
    &\P\Sp{\xi_i^c \left|\bigcap_{j=1}^{i-1} \xi_{j}\right.}
    \leq 2KT \exp\Sp{-\frac{2^{-2i}n_i^2/8}{2m\sum_{i'=1}^{i}(n_{i'}-n_{i'-1})f(\mc{A}_{i'-2})+(m\sqrt{d f(\mc{X})}+1)n_i2^{-i}/3}}\\
    &\qquad\leq 2KT\max_{k\in[K]}\exp\Sp{-\frac{3n_{\bar{i}(k)}^2\Delta_k^2}{48m\sum_{i'=1}^{\bar{i}(k)}(n_{i'}-n_{i'-1})f(\mc{A}_{i'-2})+8(m\sqrt{d f(\mc{X})}+1)n_{\bar{i}(k)}\Delta_k}},
\end{align*}
where $\bar{i}(k)=\max\Bp{i\in[i_0 - 1]\mid \Delta_k\leq 2^{-i}}$. Here, the last inequality use the same simple fact that $t\mapsto\frac{t^2}{at+b}$ is an increasing function when $t\geq 0$ if $a>0$ and $b>0$.

With values of $0=n_0 < n_1 \leq n_2 \leq \dots\leq n_{i_0}=T$, we can define $a_i=\frac{n_i}{T}$, which implies $0=a_0<a_1\leq a_2\leq\dots\leq a_{i_0}=1$. Since the choice of values $\ve{a}\in\mc{D}$ is arbitrary, the final error probability can be bounded as 
\begin{align*}
    &\P\Sp{J_T\neq 1} \leq \sum_{i=1}^{i_0}\P\Sp{ \xi_j^c\left| \bigcap_{j=1}^{i-1} \xi_j\right.}\\
    &\leq 2i_0 KT \min_{\ve{a}\in\mc{D}} \max_{k\in[K]}\exp\Sp{-\frac{3Ta_{i(k)}^2\Delta_k^2}{48m\sum_{i'=1}^{\bar{i}(k)}(a_{i'}-a_{i'-1})f(\mc{A}_{i'-2})+8(m\sqrt{d f(\mc{X})}+1)a_{i(k)}\Delta_k}},
\end{align*}
which completes the proof
\end{proof}


\subsubsection{Properties of \textsf{RAGE-Elimination}}
In this section, we prove some properties of the \textsf{RAGE-Elimination} algorithm that will be useful for proving Theorem \ref{theo:bobw_error_prob_raw}.

\begin{lemma}
\label{lem:X_in_A}
Assume $t \geq n_i$. Then, under $\bigcap_{j=1}^{i-1} \xi_{j}$, when running \textsf{RAGE-Elimination} (line \ref{algo:rage_elimination} in Algorithm \ref{algo:p1_rage}), it holds that 
$$\mc{X}_{t}^{(i+1)}\subseteq \Bp{x\in\X\mid \widehat{\Delta}^{(t)}_x \leq 2^{-i}} \subseteq \mc{A}_i.$$
\end{lemma}
\begin{proof}
To show $\mc{X}_{t}^{(i+1)}\subseteq \Bp{x\in\X\mid \widehat{\Delta}^{(t)}_x \leq 2^{-i}}$, let $x_{\widehat{(1)}_t}=\argmax_{x\in\mc{X}}\inner{\htheta_t, x}$. Then, for some arm $x$, if we have $\inner{\widehat{\theta}^{(t)}, x_{\hatone_t}-x}\leq 2^{-i}$, it holds that
$$\inner{\widehat{\theta}^{(t)}, x_1-x}=\underbrace{\inner{\widehat{\theta}^{(t)}, x_1-x_{\hatone_t}}}_{\leq 0}+\underbrace{\inner{\widehat{\theta}^{(t)}, x_{\hatone_t}-x}}_{\leq 2^{-i}}\leq 2^{-i},$$
which implies $x\in\{x\in\X\mid \widehat{\Delta}^{(t)}_x \leq 2^{-i}\}$.

To show $\Bp{x\in\X\mid \widehat{\Delta}^{(t)}_x \leq 2^{-i}} \subseteq \mc{A}_i$, let $\widehat{\Delta}^{(t)}_x \leq 2^{-i}$ for some $x$ and assume for the sake of a contradiction that $\Delta_x > 2 \cdot 2^{-i}$. As $\Delta_x > 2 \cdot 2^{-i}$, there must exist $\tilde{i} \leq i-1$ such that
$2^{-\tilde{i}} < \Delta_x \leq 2\cdot2^{-\tilde{i}}$. Then $| \Delta_x - \widehat{\Delta}^{(t)}_x| < 2^{-\tilde{i}}/2$ since event $\xi_{\tilde{i}}$ holds. Meanwhile, we have $\widehat{\Delta}^{(t)}_x \leq 2^{-i} \leq 2^{-\tilde{i}}/2$ since $\tilde{i} \leq i-1$. Now, this leads to the contradiction
$$2^{-\tilde{i}}/2 = 2^{-\tilde{i}} - 2^{-\tilde{i}}/2 \leq \Delta_x - \widehat{\Delta}^{(t)}_x \leq  | \Delta_x - \widehat{\Delta}^{(t)}_j| < 2^{-\tilde{i}}/2.$$
Thus, under $\bigcap_{j=1}^{i-1} \xi_{j}$, we have
$$\Bp{x\in\X\mid \widehat{\Delta}^{(t)}_x \leq 2^{-i}} \subseteq \Bp{x\in\X\mid \Delta_x \leq 2 \cdot 2^{-i}}=\mc{A}_i.$$

\end{proof}

\begin{lemma}
\label{lem:A_in_X}
Assume $t\geq n_i$. Then, under $\bigcap_{j=1}^{i-1} \xi_{j}$, when running \textsf{RAGE-Elimination}, if $x\in\mc{A}_{i}$, then $x\in\mc{X}_{t}^{(i-1)}$.
\end{lemma}
\begin{proof}
If $x\in\mc{A}_i$, then $\inner{\theta, x_1-x}\leq 2\cdot 2^{-i}$. Again, let $x_{\hatone_t}=\argmax_{x\in\mc{X}}\inner{\htheta_t, x}$ and we have
\begin{align*}
    \inner{\htheta_t, \hat{x}_1^{(t)}-x}=&\inner{\htheta_t, x_{\hatone_t}-x_1}+\inner{\htheta_t, x_1-x}\\
    =&\inner{\htheta_t, x_{\hatone_t}-x_1}+\inner{\htheta_t-\theta, x_1-x}+\underbrace{\inner{\theta, x_1-x}}_{\leq 2\cdot 2^{-i}}\\
    \leq & \inner{\htheta_t, x_{\hatone_t}-x_1} + |\hdeltat_x-\Delta_x| + 2\cdot 2^{-i}\\
    \leq & \inner{\htheta_t, x_{\hatone_t}-x_1} + 2^{-i} + 2\cdot 2^{-i}\tag{Since $\xi_{i-1}$ holds}\\
    = & -\hdeltat_{x_{\hatone_t}}+ 2^{-i} + 2\cdot 2^{-i}\\
    \leq & 2^{-i} + 2^{-i} + 2\cdot 2^{-i}\\
    =& 4\cdot 2^{-i}.
\end{align*}
The last inequality above holds because under $\bigcap_{j=1}^{i-1} \xi_{j}$, by Lemma \ref{lem:X_in_A}, we have $x_{\hatone_t}\in\mc{A}_i$, meaning that $|\hdeltat_{x_{\hatone_t}}-\Delta_{x_{\hatone_t}}|< 2^{-i}\implies\hdeltat_{x_{\hatone_t}}>\Delta_{x_{\hatone_t}} - 2^{-i}>-2^{-i}$.
\end{proof}

\begin{lemma}\label{lmm:subroutine_guarantees}
Assume $t\geq n_i$ and $\bigcap_{j=1}^{i-1} \xi_{j}$ holds. When running \textsf{RAGE-Elimination}, If $x_k\in\mc{A}_{i}$, then
$$\Norm{x_1-x_k}^2_{A(\bar{\lambda}_t)^{-1}}\leq m \min_{\lambda\in\triangle_{\mc{X}}}\max_{x, x'\in\mc{A}_{i-2}}\Norm{x-x'}^2_{A(\lambda)^{-1}}.$$
\end{lemma}
\begin{proof}
By Lemma \ref{lem:A_in_X}, we have $x_1, x_k\in\mc{A}_i\implies x_1, x_k\in\mc{X}_{t}^{(i-1)}$, which means that $\abs{\X_{t}^{(i-1)}}\geq 2$ and $\bar{\lambda}_t=\frac{1}{i_t}\sum_{i'=1}^{i_t}\lambda_{t}^{(i')}$ for some $i_t$ satisfying $i-1\leq i_t\leq m$. Thus, We have
\begin{align*}
    \Norm{x_1-x_k}^2_{A(\bar{\lambda}_t)^{-1}}\leq & m\Norm{x_1-x_k}^2_{A\left(\lambda_{t}^{(i-1)}\right)^{-1}}\\
    \leq& m \max_{x, x'\in\mc{X}_{t}^{(i-1)}}\Norm{x-x'}^2_{A\left(\lambda_{t}^{(i-1)}\right)^{-1}}\tag{Since $x_1, x_k\in\mc{X}_{i-1}^{(t)}$}\\
    \overset{\text{(i)}}{\leq}& m \min_{\lambda\in\triangle_{\mc{X}}}\max_{x, x'\in\mc{A}_{i-2}}\Norm{x-x'}^2_{A(\lambda)^{-1}}.
\end{align*}
Here, the above inequality (i) holds because by Lemma \ref{lem:X_in_A}, we have $\mc{X}_{t}^{(i-1)}\subseteq\mc{A}_{i-2}$ and by algorithm construction, we have $\lambda_{t}^{(i-1)}\in\argmin_{\lambda\in\triangle_{\mc{X}}}\max_{x, x'\in\mc{X}_{t}^{(i-1)}}\Norm{x-x'}^2_{A(\lambda)^{-1}}$.
\end{proof}

\subsubsection{Simplified Stationary Complexity and its Relation to Multi-armed Bandits}

In this section, we simplify the complexity of Algorithm \ref{algo:p1_rage} obtained in Theorem \ref{theo:bobw_error_prob_raw} by appropriately choosing values $\ve{a}\in\mc{D}$. In particular, we have the following theorem.
\begin{theorem}
    For $\overline{H}_{\textsf{P1-RAGE}}(\theta)$ defined in equation \eqref{equ:H_p1rage}, we have
    \fontsize{9.5}{9.5}
    $$\overline{H}_{\textsf{P1-RAGE}}(\theta)\leq \frac{1024mi_0}{\Delta_1}\inf_{\lambda\in\triangle_{\X}}\max_{x\neq x_1}\frac{\Norm{x-x_1}^2_{A(\lambda)^{-1}}}{\Delta_x} + \frac{16m\sqrt{d}}{3\Delta_1}\inf_{\lambda\in\triangle_{\X}}\max_{x\neq x_1}\Norm{x-x_1}_{A(\lambda)^{-1}}+\frac{1}{3\Delta_1}.$$
    \normalsize
\end{theorem}
\begin{proof}
    For $i\in\Bp{1, \dots, i_0-1}$, we take $a_i=\frac{\Delta_1}{\Delta_{\bar{k}(i)}}$, where $\bar{k}(i)=\min\Bp{k\in[K]\mid \Delta_k\geq\frac{2^{-i}}{2}}$. Then, since $\bar{i}(k)=\max\Bp{i\in[i_0-1]\mid \Delta_k\leq 2^{-i}}$, for any $k\in[K]$, we have $\frac{2^{-\bar{i}(k)}}{2}\leq\Delta_{\bar{k}(\bar{i}(k))}\leq\Delta_k$, which further implies 
    $$a_{\bar{i}(k)}\Delta_k=\frac{\Delta_1}{\Delta_{\bar{k}(\bar{i}(k))}}\cdot\Delta_k\geq\Delta_1.$$
    Then, for $\overline{H}_{\textsf{P1-RAGE}}(\theta)$ (defined in equation \eqref{equ:H_p1rage}), we have
    \begin{align*}
        &\overline{H}_{\textsf{P1-RAGE}}(\theta)\leq \max_{k\in[K]}\Bp{\frac{16m\sum_{i'=1}^{\bar{i}(k)}(a_{i'}-a_{i'-1})f(\mc{A}_{i'-2}) }{a_{\bar{i}(k)}^2\Delta_k^2} + \frac{8(m\sqrt{df(\mc{X})}+1)}{3a_{\bar{i}(k)}\Delta_k}}\\
        &\quad \leq  \frac{16m}{\Delta_1}\max_{k\in[K]}\Bp{\frac{f(\mc{A}_{-1})}{\Delta_{\bar{k}(1)}}+\sum_{i'=2}^{\bar{i}(k)}\Sp{\frac{1}{\Delta_{\bar{k}(i')}}-\frac{1}{\Delta_{\bar{k}(i'-1)}}}f(\mc{A}_{i'-2})} + \frac{8(m\sqrt{df(\mc{X})}+1)}{3\Delta_1}.\tag{Since $a_0=0$ by definition}
    \end{align*}
    For the second term, using the definition of $f(\X)$, we simply have
    \begin{align}
        \frac{8(m\sqrt{df(\mc{X})}+1)}{3\Delta_1}=&\frac{8m\sqrt{d}}{3\Delta_1}\inf_{\lambda\in\triangle_{\X}}\max_{x, x'\in\X}\Norm{x-x_1+x_1-x'}_{A(\lambda)^{-1}}+\frac{1}{3\Delta_1}\nonumber\\
        \leq & \frac{16m\sqrt{d}}{3\Delta_1}\inf_{\lambda\in\triangle_{\X}}\max_{x\neq x_1}\Norm{x-x_1}_{A(\lambda)^{-1}}+\frac{1}{3\Delta_1}.\label{equ:expand_norm_square}
    \end{align}
    For the first term, by fixing arm index $k\in[K]$ and defining $j\in\argmax_{\ell\in[\bar{i}(k)]}\frac{f(\mc{A}_{\ell-2})}{\Delta_{\bar{k}(\ell)}}$, we have
    \begin{align*}
        &\frac{f(\mc{A}_{-1})}{\Delta_{\bar{k}(1)}}+\sum_{i'=2}^{\bar{i}(k)}\Sp{\frac{1}{\Delta_{\bar{k}(i')}}-\frac{1}{\Delta_{\bar{k}(i'-1)}}}f(\mc{A}_{i'-2})\\
        =&\frac{f(\mc{A}_{\bar{i}(k)-2})}{\Delta_{\bar{k}(\bar{i}(k))}}+\sum_{i'=1}^{\bar{i}(k)-1}\frac{f(\mc{A}_{i'-2})-f(\mc{A}_{i'-1})}{\Delta_{\bar{k}(i')}}\\
        \overset{\text{(a)}}{\leq} & \frac{f(\mc{A}_{j-2})}{\Delta_{\bar{k}(j)}}\Sp{1+\sum_{i'=1}^{\bar{i}(k)-1}\frac{f(\mc{A}_{i'-2})-f(\mc{A}_{i'-1})}{f(\mc{A}_{i'-2})}}\\
        \leq & \bar{i}(k)\frac{f(\mc{A}_{j-2})}{\Delta_{\bar{k}(j)}} \tag{Since $f(\mc{A}_{i'-2})\geq f(\mc{A}_{i'-1})$}\\
        \leq & i_0\max_{\ell\in[\bar{i}(k)]}\frac{f(\mc{A}_{\ell-2})}{\Delta_{\bar{k}(\ell)}} \tag{Since $\bar{i}(k)\leq i_0$ for any $k\in[K]$}\\
        = & i_0 \max_{\ell\in[\bar{i}(k)]} \inf_{\lambda\in\triangle_{\X}}\max_{x, x'\in\mc{A}_{\ell-2}}\frac{\Norm{x-x'}^2_{A(\lambda)^{-1}}}{\Delta_{\bar{k}(\ell)}}\\
        \leq & i_0\inf_{\lambda\in\triangle_{\X}} \max_{\ell\in[\bar{i}(k)]}\max_{x, x'\in\mc{A}_{\ell-2}}\frac{\Norm{x-x'}^2_{A(\lambda)^{-1}}}{\Delta_{\bar{k}(\ell)}} \tag{By the weak duality inequality}\\
        \leq &64i_0\inf_{\lambda\in\triangle_{\X}} \max_{\ell\in[\bar{i}(k)]}\max_{x\in\mc{A}_{\ell-2}, x\neq x_1}\frac{\Norm{x-x_1}^2_{A(\lambda)^{-1}}}{16\Delta_{\bar{k}(\ell)}} \tag{By reasoning similar to equation \eqref{equ:expand_norm_square}}\\
        \overset{\text{(b)}}{\leq} & 64i_0\inf_{\lambda\in\triangle_{\X}} \max_{\ell\in[\bar{i}(k)]}\max_{x\in\mc{A}_{\ell-2}, x\neq x_1}\frac{\Norm{x-x_1}^2_{A(\lambda)^{-1}}}{\Delta_x}\\
        \leq & 64i_0\inf_{\lambda\in\triangle_{\X}}\max_{x\neq x_1}\frac{\Norm{x-x_1}^2_{A(\lambda)^{-1}}}{\Delta_x}.
    \end{align*}
    Here, the inequality (a) above holds because $f(\mc{A}_{i'-2})\geq f(\mc{A}_{i'-1})$ and by definition of $j$, we have $\frac{f(\mc{A}_{\ell-2})}{\Delta_{\bar{k}(\ell)}}\leq\frac{f(\mc{A}_{j-2})}{\Delta_{\bar{k}(j)}}$. The inequality (b) above holds because by definitions of $\bar{k}(\ell)=\min\Bp{k\in[K]\mid \Delta_k\geq\frac{2^{-i}}{2}}$ and $\mc{A}_{\ell-2}=\Bp{x\in\X\mid \Delta_x\leq 2\cdot 2^{-(\ell-2)}}$, we have $16\Delta_{\bar{k}(\ell)}\geq\Delta_x$ for any $x\in\mc{A}_{\ell-2}.$

    Therefore, by plugging the bound of both terms back, we have
    \fontsize{9}{9}
    $$\overline{H}_{\textsf{P1-RAGE}}(\theta)\leq \frac{1024mi_0}{\Delta_1}\inf_{\lambda\in\triangle_{\X}}\max_{x\neq x_1}\frac{\Norm{x-x_1}^2_{A(\lambda)^{-1}}}{\Delta_x} + \frac{16m\sqrt{d}}{3\Delta_1}\inf_{\lambda\in\triangle_{\X}}\max_{x\neq x_1}\Norm{x-x_1}_{A(\lambda)^{-1}}+\frac{1}{3\Delta_1}.$$
    \normalsize
\end{proof}

In the following corollary, we show that the above simplified complexity is in a same order (up to logarithmic factors) of $H_{\mathrm{BOB}}$ proposed in \citet{abbasi2018best}.
\begin{corollary}
    \label{coro:bobw_linear_to_mab}
    In multi-armed bandits, meaning $d=K$ and $\X=\Bp{\ve{e}_1, \dots, \ve{e}_K}$, for $H_{\textsf{P1-RAGE}}(\theta)$ (defined in equation \eqref{equ:H_bobw}), if $m=i_0$, we then have
    $$H_{\textsf{P1-RAGE}}(\theta)\leq \frac{2i_0\Sp{i_0\log(2K)+1}}{\Delta_{(1)}}\max_{k\in[K]}\frac{k}{\Delta_{(k)}}=2i_0\Sp{i_0\log(2K)+1}H_{\mathrm{BOB}}(\theta).$$
\end{corollary}
\begin{proof}
    When in multi-armed bandits, for the first term in $H_{\textsf{P1-RAGE}}(\theta)$, we have
    $$\inf_{\lambda\in\triangle_{\X}}\max_{x\neq x_{(1)}}\frac{\Norm{x-x_{(1)}}^2_{A(\lambda)^{-1}}}{\Delta_x}\leq 2\sum_{k=1}^{K}\frac{1}{\Delta_k}\leq 2\log(2K)\max_{k\in[K]}\frac{k}{\Delta_{(k)}},$$
    where the first inequality above comes from \citet{soare2014best} and the second inequality comes from \citet{audibert2010best}. For the second term in $H_{\textsf{P1-RAGE}}(\theta)$, we have
    $$\inf_{\lambda\in\triangle_{\X}}\max_{x\neq x_{(1)}}\Norm{x-x_{(1)}}_{A(\lambda)^-1{}}=\inf_{\lambda\in\triangle_{\X}}\max_{k\neq (1)}\sqrt{\frac{1}{\lambda_{(1)}}+\frac{1}{\lambda_{k}}}=\sqrt{2K}, $$
    which then gives us $\frac{\sqrt{K}\cdot\sqrt{2K}}{\Delta_{(1)}}\leq\frac{2K}{\Delta_{(1)}\Delta_{(K)}}\leq\frac{2}{\Delta_{(1)}}\max_{k\in[K]}\frac{k}{\Delta_{(k)}}$. 

    Finally, by plugging these inequalities back into $H_{\textsf{P1-RAGE}}(\theta)$ (defined in equation \eqref{equ:H_bobw}), we have
    \begin{align*}
        H_{\textsf{P1-RAGE}}(\theta)=& \frac{mi_0}{\Delta_{(1)}}\inf_{\lambda\in\triangle_{\X}}\max_{x\neq x_{(1)}}\frac{\Norm{x-x_{(1)}}^2_{A(\lambda)^{-1}}}{\Delta_x} + \frac{m\sqrt{d}}{\Delta_{(1)}}\inf_{\lambda\in\triangle_{\X}}\max_{x\neq x_{(1)}}\Norm{x-x_{(1)}}_{A(\lambda)^{-1}}\\
        \leq & \frac{2i_0^2\log(2K)}{\Delta_{(1)}}\max_{k\in[K]}\frac{k}{\Delta_{(k)}} + \frac{2i_0}{\Delta_{(1)}}\max_{k\in[K]}\frac{k}{\Delta_{(k)}}\\
        = & \frac{2i_0\Sp{i_0\log(2K)+1}}{\Delta_{(1)}}\max_{k\in[K]}\frac{k}{\Delta_{(k)}}.
    \end{align*}
\end{proof}

\subsubsection{Approximate BAI of Algorithm \ref{algo:p1_rage}}
\begin{corollary}
    \label{coro:epsilon_bai}
    Fix arm set $\X\subset\R^d$ with $\abs{\X}=K$ and budget $T$. For a stationary environment with unknown parameter $\theta$, if $m\geq i_0(\epsilon)=\ceil{\log_2\Sp{1/\epsilon}}+1$ for some $\epsilon\geq \Delta_{1}$, then there exists absolute constant $c>0$ such that the error probability of \textsf{P1-RAGE} satisfies
    $$\P_{\theta}\Sp{J_T\notin\mc{A}(\epsilon) }\leq 2i_0(\epsilon) KT\exp\Sp{-\frac{cT}{H_{\textsf{P1-RAGE}}(\theta, \epsilon)}},$$
    where $\mc{A}(\epsilon)=\Bp{x\in\X\mid \Delta_x\leq\epsilon}$ and $H_{\textsf{P1-RAGE}}(\theta, \epsilon)$ is defined as replacing $i_0$ by $i_0(\epsilon)$ in $H_{\textsf{P1-RAGE}}(\theta)$ (defined in Eq. \eqref{equ:H_p1rage}).
\end{corollary}
\begin{proof}
    The proof is the same as Theorem \ref{theo:bobw_upper_bound} through simply replacing $i_0$ by $i_0(\epsilon)$.
\end{proof}

\subsection{Non-stationary Environments}
In this section, we prove the error probability of Algorithm \ref{algo:p1_rage} in general non-stationary environments.
\begin{theorem}
    Fix time horizon $T$, arm set $\X\subset\R^d$ with $\abs{\X}=K$ and arbitrary unknown parameters $\Bp{\theta_t}_{t=1}^T$. If we run Algorithm \ref{algo:p1_rage} in this non-stationary environment and obtain $x_{J_T}$, then it holds that
    $$\P_{\otheta_T}\Sp{J_T\neq (1)}\leq K\exp\Sp{-\frac{3T\Delta_{(1)}^2}{64d}}.$$
\end{theorem}
\begin{proof}
    The proof will basically resemble the one for Theorem \ref{theo:adv_upper_bound}. In particular, by the same reasoning to obtain equation \ref{equ:g_bernstein}, we have
    $$\P\Sp{J_T\neq (1)}\leq \P\Sp{x_{(1)}^\top\htheta_T-x_{(1)}^\top\otheta_T\leq -\frac{\Delta_{(1)}}{2}} + \sum_{k=2}^{K}\P\Sp{x_{(k)}^\top\htheta_T-x_{(k)}^\top\otheta_T\geq\frac{\Delta_{(k)}}{2}},$$
    $$\text{where }\P\Sp{x_{(1)}^\top\htheta_T-x_{(1)}^\top\otheta_T\leq -\frac{\Delta_{(1)}}{2}}=\P\Sp{\sum_{t=1}^{T}x_{(1)}^\top\Sp{A(\lambda_t)^{-1}x_tr_t-\theta_t}\leq-\frac{T\Delta_{(1)}}{2}}.$$
    Since $\lambda_t=\frac{\bar{\lambda}_t+\lambda^*}{2}$ and $\lambda\mapsto\Norm{x}^2_{A(\lambda)^{-1}}$ is convex in $\lambda$, to use the Berstein's inequality for martingale differences \citep{freedman1975tail}, we have
    $$\abs{x_{(1)}^\top\Sp{A(\lambda_t)^{-1}x_tr_t-\theta_t}}\leq 2\Norm{x_{(1)}}_{A(\lambda^*)^{-1}}\Norm{x_t}_{A(\lambda^*)^{-1}}+2\leq 2d+2\leq 4d,$$
    $$\E\Mp{\Sp{x_{(1)}^\top\Sp{A(\lambda_t)^{-1}x_tr_t-\theta_t}}^2\mid \lambda_t}=\Norm{x_{(1)}}^2_{A(\lambda_t)^{-1}} \leq  2\Norm{x_{(1)}}^2_{A(\lambda^*)^{-1}}\leq 2d.$$
    Therefore, we have
    $$\P\Sp{x_{(1)}^\top\htheta_T-x_{(1)}^\top\otheta_T\leq -\frac{\Delta_{(1)}}{2}}\leq\exp\Sp{-\frac{T\Delta_{(1)}^2/8}{2d+2d\Delta_{(1)}/3}}\leq \exp\Sp{-\frac{3T\Delta_{(1)}^2}{64d}}.$$
    By applying the same inequality to other terms, we have
    $$\P\Sp{J_T\neq (1)}\leq K\exp\Sp{-\frac{3T\Delta_{(1)}^2}{64d}}.$$
\end{proof}
\section{IMPLEMENTATION DETAILS AND ADDITIONAL EXPERIMENTS}
\label{sec:experiment_details}

In this section, we provide more implementation details and additional experiment results. Experiments are executed through Python 3.10 and paralleled by a Mac M1 Pro chip with 6 cores.

First, we notice that an algorithm for stationary environments usually determines a batch of arms to pull at once during each epoch, while in non-stationary environment, the order of pulling these arms will affect the rewards it will receive. Therefore, when applying stationary algorithms (\textsf{Peace} and \textsf{OD-LinBAI}) into a non-stationary environment, we use a random permutation to determine the order of pulling for each batch of arms. 

When implementing \textsf{P1-RAGE}, to be computationally efficient, we update $\lambda_t$ in the same frequency as \textsf{P1-Peace}, which is summarized in Algorithm \ref{algo:p1_peace}. We take $m=15$ for \textsf{P1-RAGE}, which, based on Theorem \ref{theo:bobw_upper_bound}, is valid as long as $\Delta_{(1)}\geq 2^{-13}\approx 1.22\times 10^{-4}$. Furthermore, when implementing \textsf{Peace}, for simplicity, we use $\inf_{\lambda\in\triangle_{\X}}\rho(\mc{Z}, \lambda)$, defined in equation \eqref{equ:rho_z}, to replace all $\gamma(\mc{Z})$ used in \citet{katz2020empirical}. Since the paper of \textsf{OD-LinBAI} does not provide code, we implement it based on the pseudocode in \citet{yang2022minimax}. Finally, we use Frank-Wolfe algorithm with stepsize $\frac{1}{2(i+2)}$ in $i$-th iteration to solve all optimization problems in a form of $\inf_{\lambda\in\triangle_{\X}}\max_{y\in\mc{Y}}\Norm{y}^2_{A(\lambda)^{-1}}$.

As for code snippets reference, we use part of the code from \citet{katz2020empirical} to implement the rounding procedure used in \textsf{Peace}\footnote{No license information.} and part of the code from \citet{fiez2019sequential} to generate the base stationary instance for the multivariate testing example.\footnote{Under MIT License.} We also use code from \citet{xu2018fully} to preprocess the Yahoo! Webscope dataset.\footnote{No license information.}

\subsection{Additional Experiments}
\label{sec:additional_experiments}

Here, we provide experiment results on some additional examples to corroborate our theoretical findings.

\paragraph{Malicious non-stationary example} Because of the nature of arm elimination, algorithms designed for stationary environment can fail easily in some malicious non-stationary environments. Here, we pick the same $\X$ as \citet{soare2014best}'s stationary benchmark example and set $\omega=0.5$. Then, we take
$$\theta_t=\begin{cases}
    \matenv{0 & 1 & 1 & \dots & 1}^\top & \text{for }t=1, \dots, \frac{T}{3},\\
    \matenv{2 & 0 & 0 & \dots & 0}^\top & \text{for }t=\frac{T}{3}+1, \dots, T.
\end{cases}$$
We can see that the overall best arm is still $x_{(1)}=\ve{e}_1$. However, because of the $\theta_t$ in the first $1/3$ rounds, algorithms like \textsf{Peace} and \textsf{OD-LinBAI} will eliminate $\ve{e}_1$ in its initial phase; on the other hand, our algorithms will be robust to this non-stationarity. Here, we take $T=10^4$ and the results are shown in right plot of Figure \ref{fig:error_malicious}.

\begin{figure}[ht]
    \centering
    \includegraphics[width=0.5\linewidth]{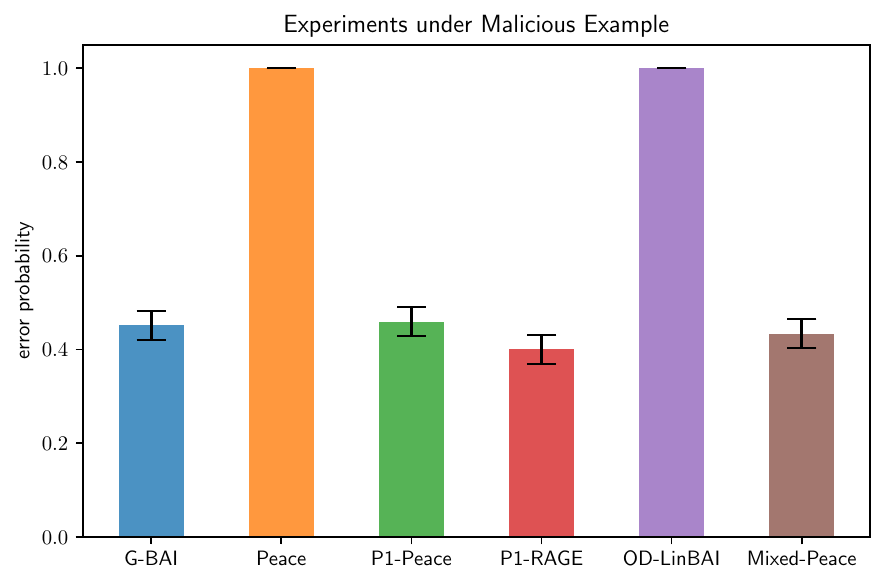}
    \caption{The error probabilities are estimated through 1000 repeated trials and the error bars represent $95\%$ confidence intervals.}
    \label{fig:error_malicious}
\end{figure}



\paragraph{Stationary multivariate testing example} We also test the performance of these algorithms in multivariate testing example when there is no non-stationarity, i.e. $\theta_t=\theta^*$ for all $t$. Here, we also take $T=10^4$ and the results are shown in Figure \ref{fig:error_multi}. We can see that our robust algorithm \textsf{P1-RAGE} again performs better than \textsf{G-BAI} and comparably with \textsf{Peace}.

\begin{figure}[ht]
    \centering
    \includegraphics[width=0.5\linewidth]{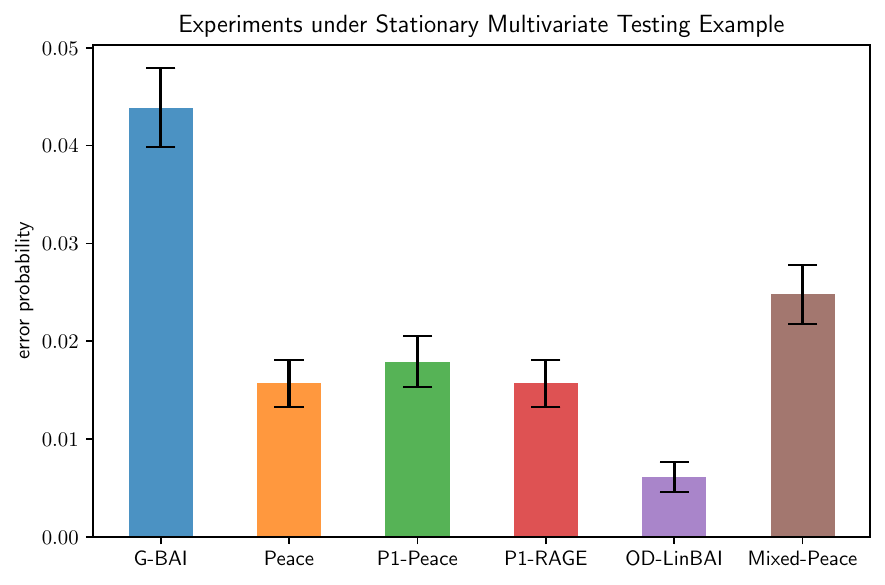}
    \caption{The error probabilities are estimated through $10^4$ repeated trials and the error bars represent $95\%$ confidence intervals.}
    \label{fig:error_multi}
\end{figure}

\paragraph{Non-stationary benchmark example} In this example, we add non-stationarity to \citet{soare2014best}'s stationary benchmark example in a more structured instead of malicious way. In particular, we keep the arm set $\X$ the same, take $\omega=0.5$ and set
$$\theta_t=\matenv{0.3 & 0 & 0 & \dots & -s\sin\Sp{\frac{2\pi t}{L}} + 0.5}^\top, $$
where $s$ is the oscillation scale and $L$ is the oscillation period, In the first series of instances, we fix $L=200$ and take values $m\in\Bp{0, 1, \dots, 9}$; in the second series of instances, we fix $m=1$ and take values $L\in\Bp{300, 600, \dots, 3000}$. All non-stationary instances have the same optimal arm as their stationary counterparts and we take $T=10^4$ for all of these instances. The results are shown in Figure \ref{fig:adv_soare}, from which we can see similar phenomenon as in Figure \ref{fig:adv_multi}. In particular, algorithms designed for stationary environments, \textsf{Peace} and \textsf{OD-LinBAI}, are very unstable in face of non-stationarity. Meanwhile, among the other four relatively robust algorithms, our algorithms \textsf{P1-RAGE} and \textsf{P1-Peace} consistently outperform the other two.

\begin{figure}[ht]
    \centering
    \includegraphics[width=\linewidth]{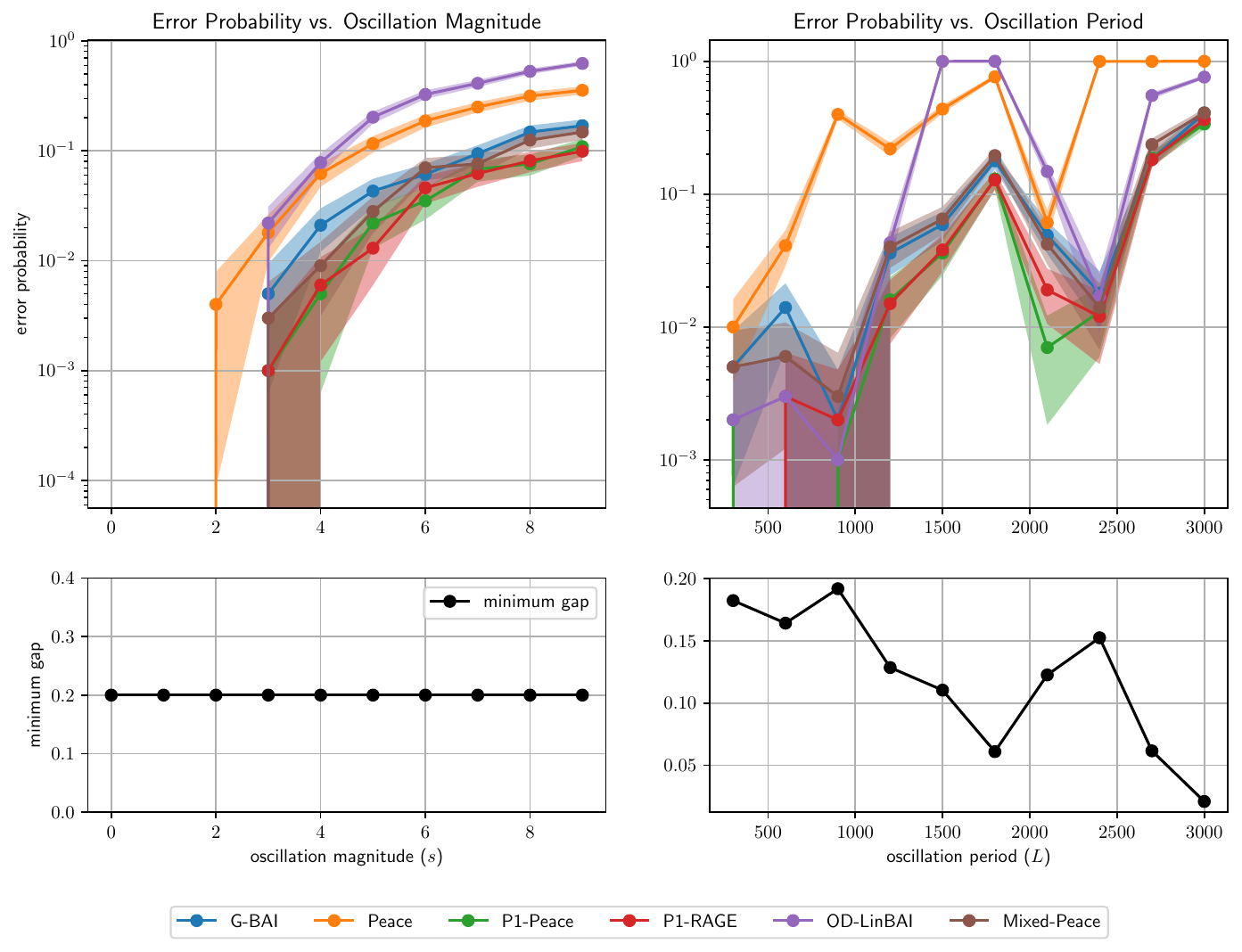}
    \caption{The vertical axis (error probability) is in log scale. The shaded area represents the $95\%$ confidence interval. Each error probability is estimated through 1000 repeated trials. The bottom two plots give the minimum gap $\Delta_{(1)}$ of each instance that algorithms run over}
    \label{fig:adv_soare}
\end{figure}

\end{document}